\pdfoutput=1
\documentclass{article}

\usepackage[main, final]{neurips_2026}

\usepackage[utf8]{inputenc}
\usepackage[T1]{fontenc}
\usepackage[hidelinks]{hyperref}
\hypersetup{pdftitle={SpecDrop: Parameter-Free Category-Conditioned Routing for Modular Specialization}, pdfauthor={Boyao Wang, Zhihan Lei}}
\usepackage{url}
\usepackage{booktabs}
\usepackage{needspace}  
\usepackage{placeins}   
\usepackage{amsfonts}
\usepackage{amsmath}
\usepackage{amssymb}
\usepackage{amsthm}
\usepackage{nicefrac}
\usepackage{microtype}
\usepackage{xcolor}
\usepackage{graphicx}
\usepackage{float}
\usepackage{subcaption}
\usepackage{algorithm}
\usepackage{algorithmic}
\usepackage{tikz}
\usepackage{multirow}
\usetikzlibrary{positioning, calc, fit, arrows.meta, decorations.pathreplacing}

\newtheorem{theorem}{Theorem}
\newtheorem{proposition}{Proposition}
\newtheorem{lemma}{Lemma}
\newtheorem{corollary}{Corollary}
\newtheorem{remark}{Remark}

\title{SpecDrop: Parameter-Free Category-Conditioned Routing for Modular Specialization}

\author{%
  Boyao Wang\thanks{Corresponding author.} \\
  Machine Learning Department\\
  Carnegie Mellon University\\
  \texttt{bryanw2@cs.cmu.edu} \\
  \And
  Zhihan Lei \\
  Machine Learning Department\\
  Carnegie Mellon University\\
  \texttt{lexl@cs.cmu.edu} \\
}

\begin{document}

\maketitle

\begin{abstract}
Modular networks such as mixture-of-experts (MoE) pursue specialization through learned routers, gates, and load-balancing losses, yet at matched total-parameter budgets learned routers can underperform equal-weight No-Routing baselines. Is the bottleneck the routing algorithm, or the alignment between training-signal granularity and the target categories? Across four settings spanning vision and language, we find the answer tracks partition granularity, not router design. We probe the question with SpecDrop, a fixed parameter-free routing scheme: each of $K$ branches receives weight $p_{\mathrm{a}}$ for its assigned category and a small leakage $p_{\mathrm{i}}{>}0$ otherwise, merged through a category-independent fixed denominator, with no learned routing parameters and no auxiliary losses; the category label is required at inference. On vision tasks where each image has one superclass label (CIFAR-100 on ResNet-110; ImageNet-1K on ViT-S/16), SpecDrop reaches $\mathbf{79.23\%}$ on CIFAR-100 and $\mathbf{79.89\%}$ on ImageNet-1K, exceeding parameter-matched baselines that do not use the label ($+4.75$ over dense on CIFAR-100; $+6.53$ over the No-Routing+SE control on ImageNet-1K). These gains quantify what category supervision buys when deployed through routing --- not an advantage over label-aware deployments of the baselines: given the same label, masking a dense model's outputs is stronger for accuracy alone ($85.2$ / $83.7$). SpecDrop's contribution is converting the label into trained-in modular structure: $68\%$/$100\%$ branch--category alignment, and masking gains of $0.00$ (CIFAR) / $+1.06$ (ImageNet) --- the output-space restriction is largely internalized during training. SpecDrop also reaches a higher top-1 than every label-free multi-branch routing baseline we evaluate at this parameter budget. On fuzzy partitions, where training units span multiple categories (SlimPajama-6B language modeling with a 30M Transformer; SuperNI instruction tuning over Llama-3.2-1B with LoRA), the routing mechanism reduces to the matched No-Routing controls within seed noise, the null our thesis predicts. Granularity alignment, not algorithm choice, localizes when routing helps. Code: \url{https://github.com/Beryex/SpecDrop}.
\end{abstract}

\section{Introduction}
\label{sec:intro}

Modular neural networks, including mixture-of-experts (MoE)~\citep{jacobs1991adaptive,jiang2024mixtral}, parallel branches~\citep{szegedy2015going,xie2017aggregated}, and multi-head attention~\citep{vaswani2017attention}, decompose model capacity across multiple parameter subsets, each meant to develop concentrated expertise on a portion of the input distribution. When achieved, specialization yields functional decomposition: distinct experts can be inspected or ablated.

Three lines of work attempt to recover specialization: \emph{(A)~Learned routing} with auxiliary losses~\citep{fedus2022switch}, \emph{(B)~Fixed-rule routing} with deterministic gates~\citep{roller2021hash}, and \emph{(C)~Dropout-style masking} at the neuron level~\citep{ba2013adaptive} or block level~\citep{fan2020reducing}. Yet specialization remains elusive: \citet{wang2026illusion} document a persistent ``standing committee'' of generalist experts across sparse MoE variants. All three intervene on the routing function or activation pattern, treating the training signal as a fixed input.

\begin{figure}[!t]
\centering
\includegraphics[width=0.9644\linewidth]{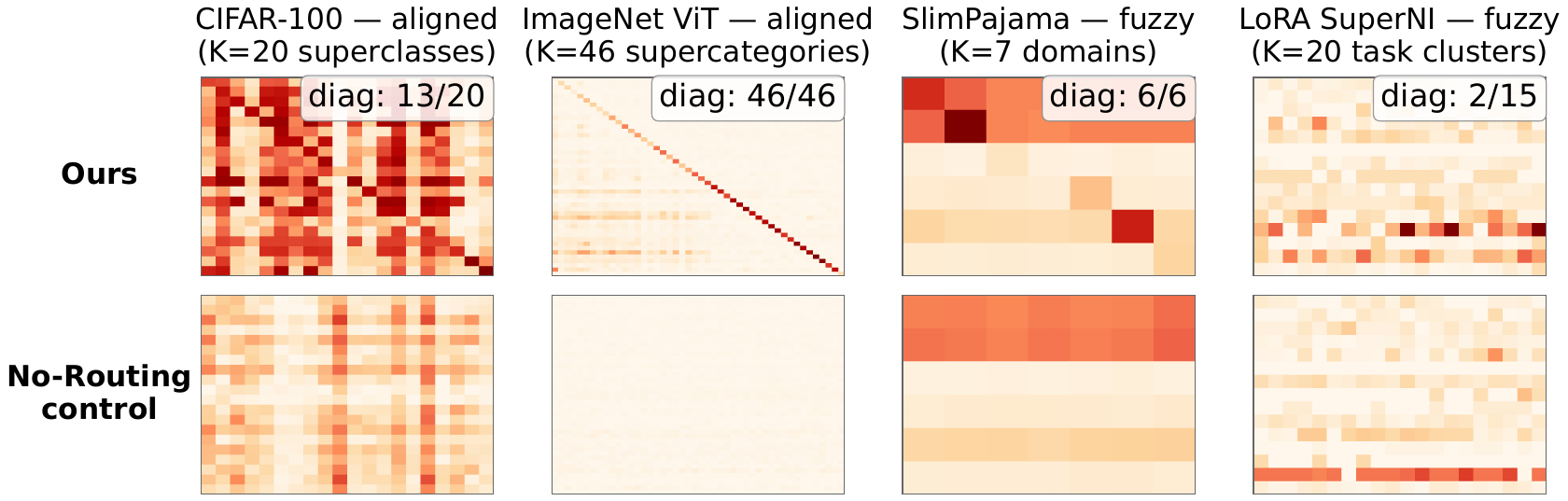}
\caption{\textbf{Per-branch pruning sensitivity across four settings: partition alignment predicts when category-conditioned routing helps.} Each panel: $x$-axis branches, $y$-axis categories; cell darkness is $|\Delta_{k,c}|$, the magnitude of the change in category-$c$ performance when branch $k$ is zero-ablated. Bottom row: each setting's architecture-matched No-Routing control ($+$SE wherever ours has one). Diag-argmax $N/M$ counts categories whose most-pruning-sensitive branch is the round-robin assigned branch (seed $42$; ties to the lowest branch index). The aligned vision partitions show a diagonal, complete on ViT BREEDS ($46/46$) and partial on CIFAR ($13/20$); SuperNI/LoRA shows little ($2/15$, over the $15$ of $20$ clusters with held-out test tasks, App.~\ref{app:lora_diag}); fuzzy SlimPajama specializes ($6/6$ domains with validation coverage; Book has none, App.~\ref{app:per_seed_domain}) without an aggregate PPL gain --- so the matched-architecture gains (Tabs.~\ref{tab:cifar_main}--\ref{tab:lora}) track partition granularity rather than diagonal counts alone (Sec.~\ref{sec:nlp}). Each column is normalized to ours' maximum, so intensities are not comparable across columns; the near-blank ViT No-Routing panel means no cell rises above noise there.}
\label{fig:specialization}
\end{figure}

As Figure~\ref{fig:specialization} shows, equal-weight No-Routing baselines fail to develop branch--category specialization across all four settings despite identical architecture. To address this gap, we propose \emph{SpecDrop}: a fixed, parameter-free routing schedule that conditions each module's activation probability on the input's category tag (the assigned module at high probability $p_{\mathrm{a}}$, the rest at a small leakage $p_{\mathrm{i}}{>}0$, merged through a category-independent denominator). It has zero learned routing parameters and zero auxiliary losses, yet produces branch--category specialization in three of the four settings.

Our claim is not that fixed routing avoids router instability, but that whether this specialization translates to performance gains depends on partition alignment: we call a partition \emph{aligned} when each training unit carries one clean category label, and \emph{fuzzy} when training units span multiple categories. On aligned vision partitions (CIFAR-100 and ImageNet-1K, where each image has one superclass), SpecDrop exceeds the parameter-matched baselines that do not use the label and reaches the highest top-1 among the multi-branch routing baselines we evaluate. These gains quantify what category supervision buys when deployed through routing, not an advantage over label-aware deployments: an information-matched masking control (Sec.~\ref{sec:discussion}) shows that, given the same label, masking a dense model's outputs is stronger for accuracy alone; what routing adds is converting the label into trained-in modular structure. On fuzzy partitions (SlimPajama-6B; SuperNI/LoRA over Llama-3.2-1B), SpecDrop reduces to matched-architecture baselines within seed noise.

\textbf{Contributions.}
\emph{(1)~SpecDrop}, the fixed, parameter-free, category-conditioned routing schedule above.
\emph{(2)~A category-independent fixed denominator $S$} (Prop.~\ref{thm:fixed_denom}) that calibrates merged-branch magnitude to single-branch scale, so the optional shared expert (SE) composes co-equally with the routed mixture.
\emph{(3)~Empirical validation and attribution}: on aligned vision partitions SpecDrop exceeds parameter-matched label-free baselines ($+4.75$ over dense on CIFAR-100, $+6.53$ over the matched No-Routing+SE on ImageNet-1K) and leads the multi-branch routing baselines we evaluate, while an information-matched masking control separates the label's share from routing's and locates SpecDrop's contribution in the induced structure; on fuzzy NLP/LoRA partitions it ties matched-architecture baselines, identifying alignment as the binding condition.

\section{Related Work}
\label{sec:related}

\paragraph{Learned routing.} Switch~\citep{fedus2022switch} uses top-1 sparse routing; V-MoE~\citep{riquelme2021vmoe} scales per-token top-$k$ routing to vision, and Soft MoE~\citep{puigcerver2024softmoe}, the same family's slot-based successor, replaces hard top-$k$ with continuous slot assignment; Expert Choice~\citep{zhou2022expert} inverts routing direction so each expert chooses its tokens; ReMoE~\citep{wang2025remoe} uses ReLU for adaptive sparsity; Mod-Squad~\citep{chen2023modsquad} adds a mutual-information loss for specialization. A parallel line progressively simplifies the routing machinery itself: StableMoE~\citep{dai2022stablemoe} freezes routing after a distillation stage because routing fluctuation harms training; ST-MoE~\citep{zoph2022stmoe} regularizes router logits for stability; auxiliary-loss-free balancing~\citep{wang2024alf} removes the balancing loss in favor of a bias correction. SpecDrop sits at the limit point of this simplification trajectory, with the assignment frozen from step~0, no router parameters, and no balancing machinery; it quantifies what this zero-fluctuation limit buys, which in our four settings is a gain only where the routing signal is category-aligned.

\paragraph{Fixed-rule routing.} Deterministic rules replace or restrict learned gates. Hash Layers~\citep{roller2021hash} hash tokens to experts; MaskMoE~\citep{su2024maskmoe} restricts a learned router with a fixed random binary mask per vocabulary token, with token frequency controlling the number of visible experts; COMET~\citep{shaier2025comet} uses fixed random projection with $k$-WTA selection. Data-domain variants route by tags: DEMix~\citep{gururangan2022demix} assigns each domain a dedicated expert, while Branch-Train variants~\citep{li2022branch,sukhbaatar2024btx} train separate experts per partition. Hash rules give up category alignment; hard domain routing ($p_{\mathrm{i}}{=}0$) severs cross-group gradient flow. SpecDrop conditions on category while keeping $p_{\mathrm{i}}{>}0$ for cross-category transfer.

\paragraph{Dropout-style masking.} Activations are stochastically dropped for regularization. Standout~\citep{ba2013adaptive}, Information/Contextual Dropout~\citep{achille2018information,fan2021contextual}, Example-Tied Dropout~\citep{maini2023example}, and NSDropout~\citep{shunk2022nsdropout} operate at the neuron level; Stochastic Depth~\citep{huang2016deep}, LayerDrop~\citep{fan2020reducing}, and SMoE-Dropout~\citep{chen2023smoe} at the block or layer level. SpecDrop instead conditions module activation on category.

\section{Method}
\label{sec:method}

SpecDrop conditions each module's activation on the input's category tag (Figure~\ref{fig:method}) through four elements: \emph{(i)}~a round-robin assignment matrix mapping categories to preferred modules, \emph{(ii)}~category-conditioned activation probabilities $p_{\mathrm{a}} > p_{\mathrm{i}} > 0$, \emph{(iii)}~a fixed-denominator merge that calibrates branch magnitudes, and \emph{(iv)}~an optional always-on shared expert.

\begin{figure}[!t]
\centering
\includegraphics[width=\linewidth]{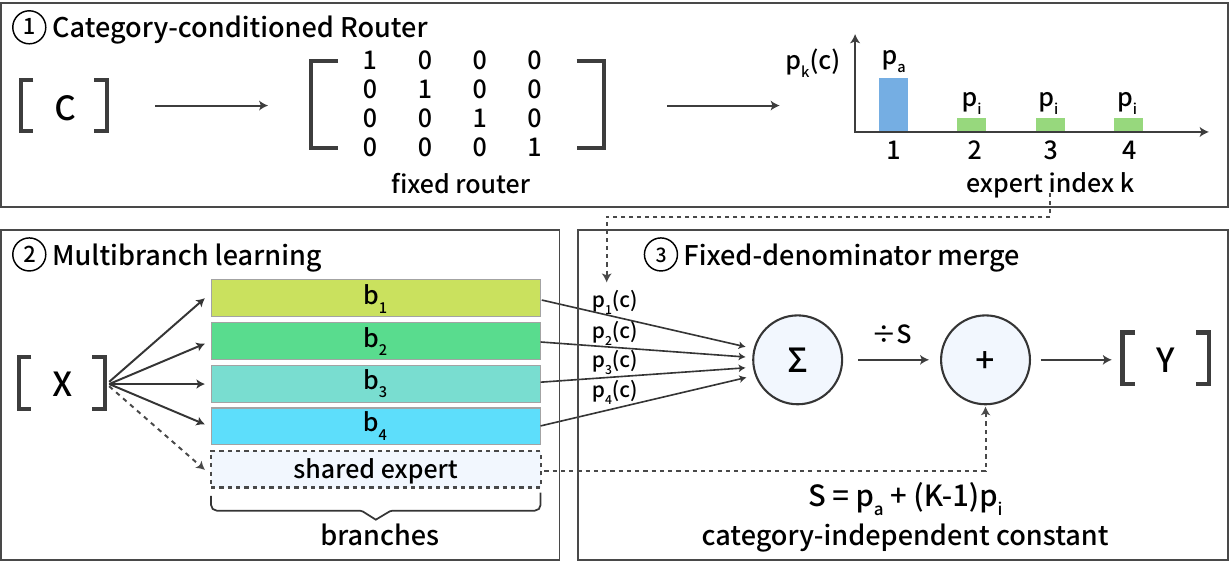}
\caption{\textbf{SpecDrop method overview.} Each input's category tag selects one preferred module via a \emph{fixed, pre-set} assignment matrix $\mathbf{A}$ (round-robin, never updated by gradient descent); the preferred module receives activation probability $p_{\mathrm{a}}$ while the rest receive a small leakage $p_{\mathrm{i}}{>}0$. Routed branches are merged through a category-independent denominator $S = p_{\mathrm{a}} + (K{-}1)p_{\mathrm{i}}$, with an optional always-on shared expert added after normalization (deployed CIFAR: $(p_{\mathrm{a}}, p_{\mathrm{i}}, K) {=} (0.7, 0.3, 20)$, merge weights $(0.109, 0.047) = (p_{\mathrm{a}}, p_{\mathrm{i}})/S$). Zero learned routing parameters, zero auxiliary losses.}
\label{fig:method}
\end{figure}

\subsection{Problem Setup}
\label{sec:setup}

Consider a deep neural network with a shared feature extractor $f_{\text{stem}}$, followed by $K$ parallel modules (branches) $\{g_k\}_{k=1}^K$, and a shared classifier head $f_{\text{head}}$.
Let $\mathcal{C} = \{1, 2, \ldots, M\}$ denote the set of data categories (e.g., superclasses in image classification, domain tags in language modeling); each training sample $(x, y)$ has an associated category label $c(x) \in \mathcal{C}$.
During standard training without routing, all $K$ modules process every input: $\hat{y} = f_{\text{head}}\bigl(\tfrac{1}{K}\sum_{k=1}^K g_k(f_{\text{stem}}(x))\bigr)$.

\subsection{Category-Conditioned Modular Dropout}
\label{sec:specdrop}

\paragraph{Assignment matrix.}
We define a binary assignment matrix $\mathbf{A} \in \{0, 1\}^{M \times K}$, where $A_{ck} = 1$ indicates that module $k$ is \emph{assigned} to category $c$.
In the simplest case, with 1-based indexing $c \in \{1, \ldots, M\}$ and $k \in \{1, \ldots, K\}$, we use round-robin assignment: $A_{ck} = \mathbf{1}[((c-1) \bmod K) + 1 = k]$, which gives each module $\lfloor M/K \rfloor$ or $\lceil M/K \rceil$ assigned categories and an exactly equal $M/K$ when $K \mid M$ (the case in our deployed settings, $K{\in}\{20,46,7,20\}$).

\paragraph{Activation probabilities.}
For a training sample with category $c$, the activation probability of module $k$ is:
\begin{equation}
    p_k(c) = A_{ck} \cdot p_{\text{active}} + (1 - A_{ck}) \cdot p_{\text{inactive}},
    \label{eq:activation_prob}
\end{equation}
where $p_{\text{active}} \in (0, 1]$ is the keep probability for assigned modules, and $p_{\text{inactive}} \in [0, p_{\text{active}})$ is the keep probability for unassigned modules (abbreviated $p_{\mathrm{a}}$ and $p_{\mathrm{i}}$).
In the Stochastic variant, each module $k$ is independently activated during training via $m_k \sim \text{Bernoulli}(p_k(c))$; at inference, and throughout training in the deployed Soft variant (Sec.~\ref{sec:soft_specdrop}), $m_k = p_k(c)$ deterministically.
The nonzero $p_{\text{inactive}} > 0$ preserves cross-category gradient flow rather than hard isolation; Theorem~\ref{thm:gradient} formalizes the resulting gradient concentration.

\subsection{Fixed-Denominator Merge}
\label{sec:fixed_denom}

To combine the $K$ branch outputs $\{h_k = g_k(f_{\text{stem}}(x))\}_{k=1}^K$, we use a \emph{fixed denominator} $S$:
\begin{equation}
    \text{output} = \frac{\sum_{k=1}^K m_k \cdot h_k}{S}, \qquad S = \sum_{k=1}^K p_k(c) = p_{\text{active}} + (K-1) \cdot p_{\text{inactive}},
    \label{eq:fixed_merge}
\end{equation}
where $S$ is \emph{category-independent} under round-robin assignment.
The fixed denominator serves a dual role.
\textbf{(i)~Train--test consistency.} Proposition~\ref{thm:fixed_denom} shows exact match at the merge layer, in contrast to the naive stochastic denominator $\sum_k m_k$ whose Jensen bias $\mathbb{E}[\sum_k m_k h_k / \sum_k m_k] \neq \sum_k p_k h_k / \sum_k p_k$ introduces train--test mismatch (a Bernoulli-variant concern, App.~\ref{app:stoch_soft}; the deployed Soft variant relies on \emph{(ii)}).
\textbf{(ii)~Magnitude calibration.} The convex-combination weights place the merged output at single-branch scale, allowing the shared expert to be added co-equally (Sec.~\ref{sec:soft_specdrop}).

\subsection{Canonical Instantiation: Soft SpecDrop}
\label{sec:soft_specdrop}

The activation probabilities $p_k(c)$ admit two instantiations: \emph{Stochastic SpecDrop} samples $m_k \sim \mathrm{Bernoulli}(p_k(c))$ (Appendix~\ref{app:stoch_soft}), while the deployed \emph{Soft SpecDrop} uses $p_k(c)$ directly as deterministic soft weights at both training and inference (no mask sampling); we retain ``SpecDrop'' for the family name while emphasizing that the canonical variant is fixed category-conditioned gating, not stochastic dropout.
We adopt Soft SpecDrop as our canonical variant and augment it with an \emph{always-on shared expert (SE)} $g_{\mathrm{SE}}$, a module disjoint from the $K$ routed branches whose output $h_{\mathrm{SE}} = g_{\mathrm{SE}}(f_{\mathrm{stem}}(x))$ is added \emph{after} the fixed-denominator normalization:
\begin{equation}
    \text{output} = \frac{\sum_{k=1}^K p_k(c)\, h_k}{S} \;+\; h_{\mathrm{SE}}.
    \label{eq:soft_specdrop}
\end{equation}
Soft SpecDrop removes Bernoulli mask-sampling variance, so the $p_k(c)/S$ scaling factor in Theorem~\ref{thm:gradient} is exact at training time (the $p_{\mathrm{a}}/p_{\mathrm{i}}$ ratio still requires the stated base-gradient symmetry); setting $g_{\mathrm{SE}}{\equiv}0$ recovers the pure routed-branches merge (CIFAR-100 config), while the always-on SE is retained on the imbalanced settings (ImageNet, SlimPajama, SuperNI).

\subsection{Per-Category Routing}
\label{sec:per_category}
\label{sec:branch_allocation}

Under imbalanced category frequencies $\{\pi_c\}$ ($\sum_c \pi_c{=}1$) we generalize $(p_{\mathrm{a}}, p_{\mathrm{i}})$ to per-category $(p_{\mathrm{a}}^c, p_{\mathrm{i}}^c)$ with $\mathrm{gap}_c = (p_{\mathrm{a}} - p_{\mathrm{i}})\cdot[(1{-}\pi_c)/(1{-}1/M)]^{\beta}$, where $\beta \geq 0$ controls how strongly the per-category gap is amplified for rare categories (larger $\beta$ = more amplification):
\begin{equation}
    p_{\mathrm{a}}^c = \tfrac{S}{K} + \mathrm{gap}_c \cdot \tfrac{K-1}{K}, \qquad
    p_{\mathrm{i}}^c = \tfrac{S}{K} - \tfrac{\mathrm{gap}_c}{K}, \qquad
    S^c \;\triangleq\; p_{\mathrm{a}}^c + (K{-}1)\, p_{\mathrm{i}}^c \;=\; S \;\;\forall\, c,\beta,\pi_c
    \label{eq:per_category}
\end{equation}
(invariance proof App.~\ref{app:per_cat_proof}).
Eq.~\ref{eq:fixed_merge} and Prop.~\ref{thm:fixed_denom} apply \emph{verbatim}; only the per-sample weights change.
At $\beta{=}0$ or in balanced settings ($\pi_c{=}1/M$), Eq.~\ref{eq:per_category} reduces to the scalar form $(p_{\mathrm{a}}, p_{\mathrm{i}})$.

\subsection{Training Schedule}
\label{sec:warmup_alignment}

$p_{\mathrm{a}}$ ramps from $S/K$ (yielding a uniform $1/K$ merge at warmup start) to its target value via a cosine schedule over a fraction $w_r \in [0, 1]$ of total training; $p_{\mathrm{i}}$ follows the coupled inverse $p_{\mathrm{i}}(t)=(S-p_{\mathrm{a}}(t))/(K-1)$ to keep $S$ constant. The warmup is applied per-epoch for vision settings and per-step for language settings, matching the learning-rate decay schedule.

\section{Theoretical Analysis}
\label{sec:theory}

We state two formal results on SpecDrop's specialization properties; a third (Thm.~\ref{thm:mi}, routing-indicator mutual information of the Bernoulli variant) and all proofs are in Appendix~\ref{app:proofs}.

\begin{theorem}[Gradient Concentration]
\label{thm:gradient}
Under SpecDrop (Eq.~\ref{eq:activation_prob}), the expected gradient magnitude for module $k$ on an assigned category $c$ ($A_{ck}{=}1$) versus an unassigned category $c'$ ($A_{c'k}{=}0$) satisfies $\mathbb{E}[\|\partial\mathcal{L}/\partial\theta_k\|\mid c]\,/\,\mathbb{E}[\|\partial\mathcal{L}/\partial\theta_k\|\mid c'] = p_{\textup{active}}/p_{\textup{inactive}}$, assuming (i)~the base gradient is independent of the mask given the input (\emph{mask-independence}), and (ii)~all categories produce equal-magnitude base gradients (\emph{category symmetry}).
\end{theorem}

\begin{proposition}[Fixed vs.\ Stochastic Denominator]
\label{thm:fixed_denom}
Let $S = p_{\textup{active}} + (K-1)p_{\textup{inactive}}$; under round-robin, $S = \sum_k p_k(c)$ is \emph{category-independent}.
\textbf{(a)}~The fixed-denominator merge matches train and test exactly: $\mathbb{E}_m[\sum_k m_k h_k / S] = \sum_k p_k(c) h_k / S$ for all $\{h_k\}$ and $c$, with weights $(p_k(c)/S)_k$ forming a convex combination.
\textbf{(b)}~The naive stochastic merge $\sum_k m_k h_k / \sum_k m_k$ of the Bernoulli variant departs from the test-time forward pass (decomposition in App.~\ref{app:proofs}).
\end{proposition}

\begin{figure}[!htbp]
\centering
\begin{tikzpicture}[font=\small, >=Stealth]
  \node (loss) at (2.55, 1.70) [draw, rounded corners, minimum width=2.0cm, minimum height=0.55cm] {loss $\mathcal{L}$};
  \node (b1) at (0.0, 0) [draw, fill=blue!12, minimum width=1.05cm, minimum height=0.6cm] {$g_1$};
  \node (b2) at (1.7, 0) [draw, minimum width=1.05cm, minimum height=0.6cm] {$g_2$};
  \node (b3) at (3.4, 0) [draw, minimum width=1.05cm, minimum height=0.6cm] {$g_3$};
  \node (b4) at (5.1, 0) [draw, minimum width=1.05cm, minimum height=0.6cm] {$g_4$};
  \draw[->, line width=1.8pt, blue!60!black] (loss.south) -- (b1.north);
  \draw[->, line width=0.6pt, gray] (loss.south) -- (b2.north);
  \draw[->, line width=0.6pt, gray] (loss.south) -- (b3.north);
  \draw[->, line width=0.6pt, gray] (loss.south) -- (b4.north);
  \node[blue!60!black, anchor=east] at (1.28, 1.15) {$\nicefrac{p_{\mathrm{a}}}{S}$};
  \node[gray!40!black, anchor=west] at (4.00, 1.15) {$\nicefrac{p_{\mathrm{i}}}{S}$};
  \node[anchor=north, gray!50!black] at (2.55, -0.45) {sample of category $c$ with $A_{c1}{=}1$: assigned branch $g_1$, unassigned $g_2,\ldots,g_K$ (shown: $K{=}4$)};
\end{tikzpicture}
\caption{\textbf{The gradient-concentration mechanism (Thm.~\ref{thm:gradient}).} Per sample, the gradient to branch $k$ scales with its merge weight $p_k(c)/S$ (for the Bernoulli variant, in expectation over the mask), so under Thm.~\ref{thm:gradient}'s assumptions an assigned-category sample sends $p_{\mathrm{a}}/p_{\mathrm{i}}$ times the expected gradient magnitude of an unassigned-category one, while the nonzero $p_{\mathrm{i}}$ preserves cross-category flow.}
\label{fig:grad_concentration}
\end{figure}

\paragraph{From theory to design.}
Thm.~\ref{thm:gradient} predicts $p_{\textup{active}}/p_{\textup{inactive}}$ as the specialization-driving gradient ratio (Fig.~\ref{fig:grad_concentration}), directionally consistent with the $1.87\times$ post-training pruning-sensitivity ratio on CIFAR (seed $42$, App.~\ref{app:mi_routing}). Thm.~\ref{thm:mi} (App.~\ref{app:proofs}) grounds the Bernoulli variant (App.~\ref{app:stoch_soft}) information-theoretically; Soft specialization is measured empirically by pruning sensitivity.

\section{Experiments}
\label{sec:experiments}

We evaluate SpecDrop across four settings spanning vision and language; we call a partition \emph{balanced} when training counts are equal across categories and \emph{imbalanced} otherwise. All numbers are mean$\pm$std over 3 seeds (42, 123, 456). Alongside each setting's primary metric we report \emph{branch--category alignment} (Align, \%): the percentage of categories whose most-pruning-sensitive branch coincides with its assigned branch. Align is a property of the trained model, distinct from the data-side partition-alignment axis of Sec.~\ref{sec:intro}: a partition can be fuzzy while Align is high, as on SlimPajama. Per-baseline protocols and the BREEDS construction are in App.~\ref{app:baseline_adaptations}, ViT and LoRA baseline notes and the ROUGE-L selection metric in App.~\ref{app:vit_miniablation} and~\ref{app:lora_miniablation}, and branch derivations in App.~\ref{app:arch}.

\subsection{CIFAR-100 Results}
\label{sec:cifar_results}

\paragraph{Setup.}
We evaluate on CIFAR-100~\citep{krizhevsky2009cifar} using its $M{=}K{=}20$ balanced superclasses on ResNet-110~\citep{he2016deep} ($\sim$1.7M params), comparing Soft SpecDrop against dense ResNet-110, Stochastic Depth~\citep{huang2016deep}, Example-Tied Dropout~\citep{maini2023example}, Contextual Dropout~\citep{fan2021contextual}, and an architecture-matched No-Routing variant. All methods train with SGD at learning rate $0.1$ on a cosine schedule, batch $128$, for $200$ epochs.

\begin{table}[!htbp]
\centering
\small
\caption{CIFAR-100 baseline comparison. \emph{HardCategory} ($p_{\mathrm{a}}{=}1, p_{\mathrm{i}}{=}0$, DEMix-style, no SE or $p_{\mathrm{a}}$ warmup) probes soft routing's $p_{\mathrm{i}}{>}0$ contribution; both Soft SpecDrop and HardCategory consume the $20$-superclass label at inference. Align chance level is $5\%$ at $K{=}20$; Align ties (an ablation often zeroes a superclass's accuracy) go to the lowest branch index, as in all tables (App.~\ref{app:mi_routing}).}
\label{tab:cifar_main}
\begin{tabular*}{\textwidth}{@{\extracolsep{\fill}}lcccc@{}}
\toprule
\textbf{Method} & \textbf{Backbone} & \textbf{\#Params} & \textbf{Top-1 (\%) $\uparrow$} & \textbf{Align (\%) $\uparrow$} \\
\midrule
ResNet-110              & \multirow{4}{*}{dense ResNet-110}     & 1.737M & $74.48 \pm 0.13$ & --- \\
Stochastic Depth        &                                       & 1.737M & $75.80 \pm 0.16$ & --- \\
Example-Tied Dropout    &                                       & 1.737M & $63.68 \pm 0.97$ & --- \\
Contextual Dropout      &                                       & 1.764M & $70.25 \pm 0.28$ & --- \\
\midrule
No-Routing              & \multirow{3}{*}{MultiBranch $K{=}20$} & 1.721M & $63.08 \pm 0.04$ & $3.3 \pm 2.9$ \\
HardCategory            &                                       & 1.721M & $57.67 \pm 3.48$ & $\mathbf{100.0 \pm 0.0}$ \\
\textbf{Soft SpecDrop}  &                                       & 1.721M & $\mathbf{79.23 \pm 0.17}$ & $68.3 \pm 5.8$ \\
\bottomrule
\end{tabular*}
\end{table}

\paragraph{Results.}
Soft SpecDrop achieves $\mathbf{79.23 \pm 0.17}\%$ (Tab.~\ref{tab:cifar_main}), exceeding dense ResNet-110 by $+4.75$ and Stochastic Depth by $+3.43$.

\paragraph{Routing-isolated comparison and HardCategory ablation.}
\label{sec:ablation_main}
The architecture-matched No-Routing baseline at $K{=}20$ ($63.08\%$, constrained by narrow-branch capacity, channels $[4,7,14]$) isolates the routing contribution under the partition-aligned regime: Soft SpecDrop's $\mathbf{+16.15}$ over it comes from category-conditioned gating (No-Routing shares the merge). HardCategory uses one-hot routing instead of soft on the same metadata, collapsing to $57.67\pm 3.48$ ($5.41$ below No-Routing, $20\times$ ours' std); our Align reaches $68.3\%$ (vs $5\%$ chance, $3.3\%$ No-Routing) while HardCategory's tautological $100\%$ co-occurs with the worst top-1, pointing to a cross-category pathway, not metadata access or alignment, as the active ingredient: with an always-on SE, hard routing ($p_{\mathrm{a}}{=}1$) reaches $76.95$ (App.~\ref{app:ablation_full}). Baseline implementations in App.~\ref{app:baseline_adaptations}.

\subsection{ImageNet ViT Results}
\label{sec:vit}

\paragraph{Setup.}
We evaluate on ImageNet-1K~\citep{deng2009imagenet} partitioned into $M{=}46$ imbalanced supercategories via our recursive expansion of the BREEDS~\citep{santurkar2021breeds} curated WordNet hierarchy (algorithm, hyperparameters, and group sizes in App.~\ref{app:breeds_construction}). All methods are based on ViT-Small/16~\citep{dosovitskiy2021vit, touvron2021deit} ($\sim$22M params). We compare Soft SpecDrop against dense ViT-S/16, Soft MoE~\citep{puigcerver2024softmoe} in tuned and compute-matched variants (Tab.~\ref{tab:vit}), an auxiliary-loss-free (ALF) top-$k$ router~\citep{wang2024alf}, Mod-Squad~\citep{chen2023modsquad} (FFN-only, adapted to ImageNet BREEDS), COMET~\citep{shaier2025comet}, and an architecture-matched No-Routing+SE baseline. All methods train with AdamW at learning rate $2.5{\times}10^{-4}$ (the tuned Soft MoE at its swept $5{\times}10^{-4}$) on a cosine schedule, batch $256$, for $100$ epochs of the DeiT short recipe (no EMA, no RepeatedAugmentation).

\paragraph{Results.}
Soft SpecDrop reaches $\mathbf{79.89 \pm 0.18}\%$ top-1 on ImageNet-1K BREEDS (Tab.~\ref{tab:vit}), $+6.53$ over the architecture-matched No-Routing+SE baseline and $+3.20$ over the strongest learned router, the tuned Soft MoE ($76.69 \pm 0.70$, itself above dense); the SE contributes $+2.06$. The compute-matched Soft MoE (dense-level MACs, parameters unconstrained) lands at $66.72$, below the bare No-Routing control: the all-blocks Soft MoE's constraint was placement, not compute (App.~\ref{app:vit_miniablation}).

\begin{table}[!htbp]
\centering
\small
\caption{ImageNet-1K BREEDS-46 on MultiBranch ViT-Small/16 (henceforth MB-ViT), with every baseline's parameter and compute status stated per row. \emph{Parameter status} vs dense ViT-S ($22.051$M): tuned Soft MoE $+0.7\%$, ALF router $+0.6\%$, Mod-Squad $+1.0\%$, COMET $+0.0\%$, No-Routing $+0.2\%$, No-Routing$+$SE and ours $+1.0\%$, all within $\pm 1\%$; the compute-matched Soft MoE's parameters are \emph{unconstrained by design} ($481$M, $21.8\times$ dense). \emph{Compute status} is the MACs column (App.~\ref{app:flops}): ours and the No-Routing$+$SE control are compute-identical ($4.25$G, matching dense); the deployed all-blocks Soft MoE ($63.06 \pm 0.22$, retained as an appendix ablation, App.~\ref{app:vit_miniablation}) runs at $1.56$G under parameter matching, and the compute-matched variant restores dense-level MACs. Soft MoE rows use a dedicated SoftMoEViT implementation (\citealp{puigcerver2024softmoe}, Alg.~1); the tuned row uses the paper-canonical second-half placement and lr $5{\times}10^{-4}$ from a dedicated placement/learning-rate sweep (App.~\ref{app:vit_miniablation}), the only such search any method received. The ALF router~\citep{wang2024alf} is capacity-identical to Mod-Squad, isolating the balancing mechanism. Only ours consumes the (target-derived) BREEDS label; No-Routing$+$SE is the architecture- and compute-matched control (Sec.~\ref{sec:vit}).}
\label{tab:vit}
\resizebox{\textwidth}{!}{%
\begin{tabular}{@{}lccccc@{}}
\toprule
\textbf{Method} & \textbf{Backbone} & \textbf{\#Params} & \textbf{MACs (G)} & \textbf{Top-1 (\%) $\uparrow$} & \textbf{Align (\%) $\uparrow$} \\
\midrule
ViT-Small/16              & Dense ViT-S/16                                       & 22.051M & 4.25 & $76.38 \pm 0.15$ & --- \\
COMET                     & Dense ViT-S/16 $+$ $k$-WTA mask                      & 22.051M & 5.65 & $71.02 \pm 0.17$ & --- \\
Soft MoE (tuned)          & SoftMoEViT $N{=}32$ 2nd-half                         & 22.196M & 2.91 & $76.69 \pm 0.70$ & --- \\
Soft MoE (comp.-matched)  & SoftMoEViT $N{=}32$ wide                             & 481.2M  & 4.26 & $66.72 \pm 0.63$ & --- \\
ALF top-$k$ router        & \multirow{2}{*}{MoE-ViT $N{=}16$, top-$2$}           & 22.194M & 4.27 & $71.09 \pm 0.32$ & --- \\
Mod-Squad                 &                                                      & 22.267M & 4.27 & $70.11 \pm 0.61$ & $1.4 \pm 1.3$ \\
\midrule
No-Routing                & MultiBranch $K{=}46$                                 & 22.092M & 4.22 & $71.30 \pm 0.22$ & $5.1 \pm 4.5$ \\
\midrule
No-Routing                & \multirow{2}{*}{MultiBranch $K{=}46$ $+$ SE}         & 22.263M & 4.25 & $73.36 \pm 0.29$ & $2.9 \pm 2.5$ \\
\textbf{Soft SpecDrop}    &                                                      & 22.263M & 4.25 & $\mathbf{79.89 \pm 0.18}$ & $\mathbf{100.0 \pm 0.0}$ \\
\bottomrule
\end{tabular}}
\end{table}

\paragraph{Routing-isolated comparison and BREEDS label-leak.}
The proper isolated baseline is MB-ViT No-Routing$+$SE (identical architecture, parameters, and compute; Tab.~\ref{tab:se_routing_decomp}), and pruning sensitivity confirms the $\mathbf{+6.53}$ over it is mechanism-driven: ours $100\%$ Align (vs $2.2\%$ chance at $K{=}46$); matched-SE No-Routing $2.9\%$ --- partition-aligned specialization, not capacity. The $+3.51$ over dense also spans an architecture change; neither baseline consumes the BREEDS label, which is constructed from fine labels~\citep{santurkar2021breeds}, so both margins include the label's value (Sec.~\ref{sec:discussion}).

\subsection{NLP Domain-Conditioned Routing}
\label{sec:nlp}

\paragraph{Setup.}
We evaluate on SlimPajama-6B~\citep{soboleva2023slimpajama} language modeling over its $M{=}K{=}7$ imbalanced document domains (${\sim}15\times$ max/min ratio), all methods using a 6-layer / 384-hidden / 6-head Transformer LM ($\sim$30M params). We compare Soft SpecDrop against the dense Transformer~\citep{vaswani2017attention}, Switch~\citep{fedus2022switch}, Hash Layers~\citep{roller2021hash}, SMoE-Dropout~\citep{chen2023smoe}, DEMix~\citep{gururangan2022demix}, and an architecture-matched No-Routing$+$SE baseline. All methods train with AdamW at learning rate $3{\times}10^{-4}$ on a cosine schedule in bf16, batch $32$ sequences of $512$ tokens, for $10$ epochs over $500$M unique tokens.

\paragraph{Results.}
Table~\ref{tab:nlp} reports validation perplexity. Soft SpecDrop achieves $\mathbf{45.38 \pm 0.02}$ PPL; ours trails matched-SE No-Routing by $+0.10$ PPL (per-seed deltas $+0.05$/$+0.02$/$+0.22$, App.~\ref{app:per_seed_domain}), within the seed-noise envelope. The always-on shared expert contributes $1.52$ PPL (Tab.~\ref{tab:se_routing_decomp}); ours nonetheless achieves the lowest perplexity among learned and fixed-rule routers.

\begin{table}[!htbp]
\centering
\small
\caption{SlimPajama-6B 7-domain language modeling on a 30M Transformer LM. SMoE-Dropout trains with one active expert and evaluates with all $16$ (the gradual-$k$ schedule of \citet{chen2023smoe} was not applied; App.~\ref{app:baseline_adaptations}). PPL std is the population form ($N$ denominator); every other std in Tabs.~\ref{tab:cifar_main}--\ref{tab:lora}, including Align here, is the sample form ($N{-}1$).}
\label{tab:nlp}
\begin{tabular*}{\textwidth}{@{\extracolsep{\fill}}lcccc@{}}
\toprule
\textbf{Method} & \textbf{Backbone} & \textbf{\#Params} & \textbf{Val PPL $\downarrow$} & \textbf{Align (\%) $\uparrow$} \\
\midrule
Dense Transformer        & Dense Transformer LM                              & 30.143M & $\mathbf{44.80 \pm 0.05}$ & --- \\
\midrule
Hash Layers              & MultiBranch $K{=}8$                               & 30.159M & $52.05 \pm 0.06$ & $44.4 \pm 9.6$ \\
SMoE-Dropout             & MultiBranch $K{=}16$                              & 30.214M & $67.32 \pm 0.50$ & $5.6 \pm 9.6$ \\
Switch                   & MultiBranch $K{=}32$                              & 30.288M & $49.54 \pm 0.20$ & $0.0 \pm 0.0$ \\
\midrule
DEMix                    & \multirow{2}{*}{MultiBranch $K{=}7$}              & 30.175M & $53.31 \pm 0.08$ & $\mathbf{100.0 \pm 0.0}$ \\
No-Routing               &                                                   & 30.175M & $46.80 \pm 0.11$ & $11.1 \pm 9.6$ \\
\midrule
No-Routing               & \multirow{2}{*}{MultiBranch $K{=}7$ $+$ SE}       & 30.168M & $45.28 \pm 0.10$ & $5.6 \pm 9.6$ \\
\textbf{Soft SpecDrop}   &                                                   & 30.168M & $45.38 \pm 0.02$ & $94.4 \pm 9.6$ \\
\bottomrule
\end{tabular*}
\end{table}

\paragraph{Routing-isolated tie and granularity alignment.}
SlimPajama's $7$ document-level domains form a fuzzy partition: over the full validation set ($9{,}766$ chunks), $56.1\%$ of $512$-token chunks span ${\geq}2$ $k$-means clusters of BGE~\citep{xiao2023cpack} embeddings (App.~\ref{app:nlp_clusters}); the tie with the mechanism-OFF reference is the predicted outcome under our granularity-alignment thesis. The misalignment operates at training time and at distribution level: per-chunk purity is uncorrelated with the per-chunk cross-entropy difference against the matched control (Pearson $r{=}{-}0.001$, $p{=}0.95$, App.~\ref{app:nlp_clusters}), consistent with branch specialization forming over the whole training distribution rather than per evaluation chunk. Pruning sensitivity rules out mechanism failure: ours $94.4\%$ Align (vs $14.3\%$ chance, $5.6\%$ matched-SE) --- the tied PPL reflects partition fuzziness, not mechanism failure to engage. Within the tie, our standard error $\sigma_{\text{mean}}{=}0.014$ (std$/\sqrt{3}$) is $4\times$ tighter than the matched-SE control's $0.057$ (per-seed PPL in App.~\ref{app:per_seed_domain}).

\subsection{SuperNI Instruction-Tuning Results (Llama-3.2-1B + LoRA)}
\label{sec:lora}

\paragraph{Setup.}
We evaluate on SuperNI~\citep{wang2022supernaturalinstructions} instruction tuning over $M{=}K{=}20$ imbalanced task clusters constructed by frequency-cutoff over the SuperNI \emph{Domains} field (App.~\ref{app:superni_clusters}), evaluating on 119 held-out tasks. All methods run as adapters on a frozen Llama-3.2-1B~\citep{grattafiori2024llama3} base ($\sim$225M trainable LoRA parameters per method). We compare Soft SpecDrop against single LoRA~\citep{hu2022lora}, LoRAMoE~\citep{dou2024loramoe}, MoCLE~\citep{gou2024mocle}, HydraLoRA~\citep{tian2024hydralora}, and an architecture-matched No-Routing$+$SE baseline. All methods train with AdamW at learning rate $2{\times}10^{-4}$ on a cosine schedule, at an effective batch of $128$ ($8$ per device $\times$ $16$ accumulation steps), for $3$ epochs on the SuperNI training mix.

\paragraph{Results.}
\label{sec:lora_decomposition}
Soft SpecDrop achieves $\mathbf{0.5106 \pm 0.003}$ ROUGE-L F1 (Tab.~\ref{tab:lora}), within seed noise of HydraLoRA ($0.5153$; seed-$42$ per-task $47/47/25$ ours-wins / HydraLoRA-wins / ties, no significant cluster association, App.~\ref{app:lora_fisher_per_task}) and LoRAMoE ($0.5079$); the routing-only contribution over matched-SE No-Routing ($0.5094$) is $+0.0012$, statistically zero on this fuzzy partition. Pruning sensitivity finds only weak, seed-unstable alignment at the imposed $K{=}20$ partition (ours $22.2\%$; the other five ${\leq}6.7\%$, chance $5\%$; SlimPajama's $94.4\%$), with no gain, as our thesis predicts.

\begin{table}[!htbp]
\centering
\small
\caption{SuperNI held-out instruction tuning (119 test tasks; ROUGE-L on $10$ instances per task for greedy-decoding cost; training and CE-loss eval use the canonical $100$) on Llama-3.2-1B with LoRA adapters. \#Params is trainable adapter only; the frozen Llama base ($\sim$1.24B) is excluded.}
\label{tab:lora}
\begin{tabular*}{\textwidth}{@{\extracolsep{\fill}}lcccc@{}}
\toprule
\textbf{Method} & \textbf{Backbone} & \textbf{\#Params} & \textbf{ROUGE-L F1 $\uparrow$} & \textbf{Align (\%) $\uparrow$} \\
\midrule
Single LoRA              & Single LoRA $r{=}320$                              & 225.44M & $0.4754 \pm 0.007$ & --- \\
MoCLE                    & MultiBranch $K{=}5$                                & 224.67M & $0.4924 \pm 0.010$ & $0.0 \pm 0.0$ \\
LoRAMoE                  & MultiBranch $K{=}6$                                & 225.31M & $0.5079 \pm 0.002$ & $2.2 \pm 3.8$ \\
HydraLoRA                & MultiBranch $K{=}8$                                & 226.56M & $\mathbf{0.5153 \pm 0.003}$ & $0.0 \pm 0.0$ \\
\midrule
No-Routing               & MultiBranch $K{=}20$                               & 225.44M & $0.4993 \pm 0.011$ & $6.7 \pm 6.7$ \\
\midrule
No-Routing               & \multirow{2}{*}{MultiBranch $K{=}20$ $+$ SE}       & 221.92M & $0.5094 \pm 0.007$ & $0.0 \pm 0.0$ \\
\textbf{Soft SpecDrop}   &                                                    & 221.92M & $0.5106 \pm 0.003$ & $\mathbf{22.2 \pm 15.4}$ \\
\bottomrule
\end{tabular*}
\end{table}

\subsection{Hyperparameter Ablations}
\label{sec:nlp_miniablation}

We sweep three core hyperparameters in sequence to identify deployed operating points (Fig.~\ref{fig:ablation_grid}, App.~\ref{app:ablation_full}). \textbf{Activation probability $p_{\mathrm{a}}$} (Row~1) is searched first, with $p_{\mathrm{i}}$ coupled as $1{-}p_{\mathrm{a}}$ (all sweeps and deployed operating points obey this coupling), excluding $p_{\mathrm{a}}{=}0.5$ a priori (where $p_{\mathrm{a}}{=}p_{\mathrm{i}}$ collapses to the mechanism-OFF reference, voiding downstream search). The remaining sweep selects $p_{\mathrm{a}}{=}0.7$ on CIFAR-100, $0.6$ on ImageNet and SlimPajama, and $0.8$ on SuperNI/LoRA. On the two vision settings both $p_{\mathrm{a}}{=}0.5$ and $p_{\mathrm{a}}{=}1$ score below the selected value (both sweeps keep an SE; $p_{\mathrm{a}}{=}1$ trails by $2.77$ on ImageNet, ${\approx}0.7\sigma$ on CIFAR); on SlimPajama PPL rises monotonically in $p_{\mathrm{a}}$ (the excluded $p_{\mathrm{a}}{=}0.5$ is lowest at both SE anchors, App.~\ref{app:nlp_miniablation}), and on SuperNI/LoRA, where $p_{\mathrm{a}}{=}1$ was excluded after the sweep, all six values (including $p_{\mathrm{a}}{=}1$: $0.4786$ vs.\ $0.4784$ selected) lie within $0.0030$ ROUGE-L (App.~\ref{app:lora_miniablation}). \textbf{Warmup ratio $w_r$ and imbalance amplification $\beta$} (Row~2): on balanced CIFAR-100 the per-category formula degenerates to a scalar across all $\beta$, so we sweep $w_r$ instead and adopt $w_r{=}1.0$ (cosine LR-aligned); on the imbalanced settings $\beta{=}4$ (SlimPajama) and $\beta{=}1$ (ImageNet, SuperNI/LoRA) are selected, all $\beta$ within seed noise. \textbf{Shared-expert capacity $X$} (Row~3): imbalanced settings use an SE ($X \in [0.5, 2.0]$); balanced CIFAR-100 uses none ($X{=}0$).

\subsection{Cross-Setting Specialization Analysis}
\label{sec:specialization}
\label{sec:nlp_specialization}
\label{sec:nlp_reliance}

We measure per-branch specialization via pruning sensitivity ($\Delta_{k,c}$ from zero-ablating branch $k$), visualized as heatmaps in Figure~\ref{fig:specialization}; per-method Align values are in Tabs.~\ref{tab:cifar_main}--\ref{tab:lora}. Disabling routing at inference (uniform $1/K$ masks) ranks the four settings via the resulting $\Delta$: $-69.83$ acc on CIFAR, $-8.61$ on ImageNet, $+0.37$ PPL on SlimPajama, $+0.0014$ ROUGE on SuperNI/LoRA --- the mechanism is load-bearing on aligned partitions and vestigial on fuzzy ones.

\begin{table}[!htbp]
\centering
\small
\caption{\textbf{SE/routing decomposition across the four settings.} SE $\Delta$ is the shared expert's architectural contribution (No-Routing $\to$ $+$SE, available to every method); routing $\Delta$ is the mechanism's marginal contribution at matched architecture ($+$SE $\to$ ours, identical parameters and compute; only ours reads the label). Deployed CIFAR uses no SE ($X{=}0$, App.~\ref{app:ablation_full}), so its routing $\Delta$ is measured from bare No-Routing.}
\label{tab:se_routing_decomp}
\begin{tabular*}{\textwidth}{@{\extracolsep{\fill}}llcccc@{}}
\toprule
\textbf{Setting} & \textbf{Metric} & \textbf{No-Routing} & \textbf{$+$SE} & \textbf{Ours} & \textbf{SE $\Delta$ / routing $\Delta$} \\
\midrule
CIFAR-100    & Top-1 $\uparrow$   & 63.08  & (no SE) & 79.23  & --- / $+16.15$ \\
ImageNet-1K  & Top-1 $\uparrow$   & 71.30  & 73.36   & 79.89  & $+2.06$ / $+6.53$ \\
SlimPajama   & PPL $\downarrow$   & 46.80  & 45.28   & 45.38  & $-1.52$ / $+0.10$ \\
SuperNI/LoRA & ROUGE-L $\uparrow$ & 0.4993 & 0.5094  & 0.5106 & $+0.0101$ / $+0.0012$ \\
\bottomrule
\end{tabular*}
\end{table}

The routing share is $+16.15$ and $+6.53$ on the aligned vision partitions, versus $0.10$ PPL behind and $+0.0012$ ROUGE-L (both within seed noise) on the fuzzy ones (Tab.~\ref{tab:se_routing_decomp}).

On SuperNI, pruning sensitivity detects only weak, seed-unstable LoRA branch alignment at the imposed $K{=}20$ ($2/15$ diagonal hits at seed $42$, $22.2\%$ over three seeds; App.~\ref{app:lora_diag}), consistent with an imposed partition that mismatches the underlying data structure. An a-priori embedding-validity diagnostic (App.~\ref{app:nlp_clusters}) corroborates the modality split from data alone: BGE text-chunk embeddings yield silhouette $0.031$ (continuous manifold) while DINOv2~\citep{oquab2024dinov2} image embeddings yield $0.069$ (mild discrete structure).

\section{Conclusion}
\label{sec:discussion}

We have argued that specialization in modular networks is governed by the alignment between training-signal granularity and the target categories, not the routing function alone. Across four settings (Tab.~\ref{tab:se_routing_decomp}), the mechanism's contribution is qualitatively positive on aligned partitions (CIFAR-100 $+16.15$, ViT BREEDS $+6.53$ over matched controls) and indistinguishable from these controls on fuzzy ones (SlimPajama-6B $0.10$ PPL behind, tied; SuperNI/LoRA $+0.0012$, tied). HardCategory (one-hot, no shared expert) collapses $5.41$ below No-Routing on CIFAR despite tautological $100\%$ Align, so a cross-category pathway is load-bearing (here the leakage $p_{\mathrm{i}}{>}0$; with an always-on SE, hard routing reaches $76.95$), not metadata access or alignment itself. Granularity alignment, not algorithm choice, localizes when routing helps.

\paragraph{What the label buys: the masking control and its scope.}\looseness=-1
An information-matched control sharpens what the vision accuracy gains mean. Given the same label at inference, simply masking a dense model's logits to its fine classes is stronger for accuracy alone: masked dense reaches $85.23$ on CIFAR-100 and $83.65$ on ImageNet, above ours at $79.23$/$79.89$ (3-seed control, App.~\ref{app:logit_mask}). Where the output space is hierarchically partitioned by the category and the deployment goal is only accuracy under a trusted label, logit masking is the stronger and cheaper mechanism at the scales we test; SpecDrop's contribution is inducing modular structure at matched architecture, and the accuracy comparisons in the paper (against No-Routing controls without inference-time masking) should be read in that scope.

Applying the same masking to SpecDrop's own outputs adds exactly $0.00$ on CIFAR (one out-of-superclass prediction in $30{,}000$ across three seeds, versus $16.3\%$ for dense) and $+1.06$ on ImageNet: the output-space restriction is largely internalized during training, and what remains beyond it is the trained-in structure of Tabs.~\ref{tab:cifar_main}--\ref{tab:vit} and Fig.~\ref{fig:specialization} (branch--category alignment and per-category pruning sensitivity) that a monolithic model, masked or not, has no substructure to support. On SlimPajama and SuperNI the output space (a shared vocabulary, free-form text) admits no category masking.

\paragraph{The matched-architecture margin as an instrument.}\looseness=-1
The $+6.53$ margin over the matched-SE control is a category-oracle reference for what category structure buys through routing at this scale, and learned routers can be scored against it. The tuned Soft MoE recovers $+3.33$ of this margin ($76.69$ vs $73.36$), almost exactly half, while SpecDrop retains a $3.20$-point margin at matched parameters; the ALF router ($71.09$) and Mod-Squad ($70.11$) remain below the $73.36$ control (App.~\ref{app:vit_miniablation}).

\paragraph{Limitations.}\looseness=-2
Three axes frame our results. On scale, our experiments span 30M language models, with a 125M 3-seed replication verifying the same direction (App.~\ref{app:scaling_check}); the $\geq 10$B Mixture-of-Experts and $\geq 1$B Chinchilla-optimal regimes remain open. On inference metadata, the category label is required at deployment (latent or overlapping categories are outside the method's scope; an ambiguous tag acts as label noise, quantified below), and on the two vision settings the label is target-derived (CIFAR-100 superclass is a coarsening of the fine label; ImageNet/BREEDS supercategories are constructed from fine labels), and no baseline consumes it, so every margin in Tabs.~\ref{tab:cifar_main}--\ref{tab:vit} includes the label's value; the matched-architecture comparisons fix the architecture, not the information.

When the label must be predicted, the label-quality curve (App.~\ref{app:label_quality}) makes the applicability condition quantitative. SpecDrop stays above the architecture-matched control down to ${\approx}80\%$ category-label accuracy on CIFAR-100, and training-free predictors (the dense model's own prediction coarsened through the hierarchy; ${\approx}2\times$ deployment cost, App.~\ref{app:label_quality}) keep it above the control on both vision settings: the $83.8\%$ CIFAR predictor retains $+6.2$ over the matched control ($69.3$ vs $63.08$), and the $88.5\%$ ImageNet predictor retains $+0.9$ ($74.25$ vs $73.36$). Beating the dense reference instead requires ${\approx}92$--$94\%$ label accuracy, which no capacity-comparable predictor reaches on CIFAR's 5-classes-per-category partition (a coarse head fine-tuned from the dense checkpoint attains ${\sim}84.5\%$); under predicted labels the vision settings do not clear that higher bar --- that too is our own measurement, and it is why the applicability condition (a category partition aligned with the training units, with a trusted tag available at inference) is load-bearing. On theory, our results characterize the construction rather than provide tight bounds.

\paragraph{Scope on dense prediction.}\looseness=-1
For segmentation, detection, or VQA, a single input contains multiple categories, so whole-input tags break the one-clean-tag alignment condition; the thesis makes a falsifiable prediction there: no gain over matched controls without region-level tags, the same reduction-to-control observed on SlimPajama and SuperNI. The alignment condition is restorable: per-pixel class labels coarsen to per-region superclass tags, and our merge already applies the mask per sample (a $(B,1,K,1)$ broadcast over tokens), so a region-level tag promotes it to per-token $(B,T,K)$, the fixed-signal analogue of V-MoE's per-token routing. Testing whether region-level tags recover the vision-classification gains is the direct falsification test of the thesis on dense prediction.

\paragraph{Impact statement.}
The granularity-alignment characterization guides when category-conditioned routing yields gains; deployments should ensure the category signal does not encode sensitive demographic or proxy attributes, since routing could amplify upstream bias.

\paragraph{Future work.}\looseness=-2
Token-level tag attribution could recover the routing signal lost to chunk-level coarsening ($56.1\%$ of $512$-token chunks span ${\geq}2$ BGE clusters, $n{=}9{,}766$, App.~\ref{app:nlp_clusters}); predictor co-training would address the inference-metadata caveat. \emph{(i)~Closed-form auxiliary-loss-free balancing} (bridging Sec.~\ref{sec:related}'s trajectory): with a fixed assignment, branch load is computable in closed form from category frequencies, and the per-category $\beta$-amplification already plays, statically, the role of \citet{wang2024alf}'s online bias correction; combining a category prior with their bias update is the natural synthesis when category structure is informative but imperfect. \emph{(ii)~Fixed-to-learned handover}: StableMoE~\citep{dai2022stablemoe} distills a learned router into a frozen one; the mirror-image curriculum (train under the fixed signal, then hand off to a learned router needing no labels at inference) would import SpecDrop's specialization into standard MoE deployment.

\clearpage
\begin{ack}
We thank the anonymous reviewers and the area chair for their feedback. \textbf{Funding:} this work received no third-party funding or support; all experiments ran on cloud GPUs rented and paid for personally by the authors. \textbf{Competing interests:} none in the past 36 months.
\end{ack}

\bibliographystyle{plainnat}
\bibliography{references}

\clearpage
\appendix

\section{Stochastic SpecDrop Formulation and Additional Ablations}
\label{app:stoch_soft}
\label{app:denom_ablation}

This appendix contains (i) the Stochastic (Bernoulli) SpecDrop variant against which the canonical Soft variant is benchmarked, (ii) a mask $\times$ denominator ablation at the ResNet-110 / CIFAR-100 scale, and (iii) a random-permutation robustness check on the assignment matrix $\mathbf{A}$.
All results support the main-text claim that the category-asymmetric \emph{deterministic} soft signal is the active ingredient; the Bernoulli variant and the specific $\mathbf{A}$ ordering are not.

\subsection{CIFAR-100 \texorpdfstring{$p_{\mathrm{a}}$}{p\_a} and SE-Ratio Sweeps}
\label{app:ablation_full}

\begin{figure}[!htbp]
\centering
\includegraphics[width=0.90\linewidth]{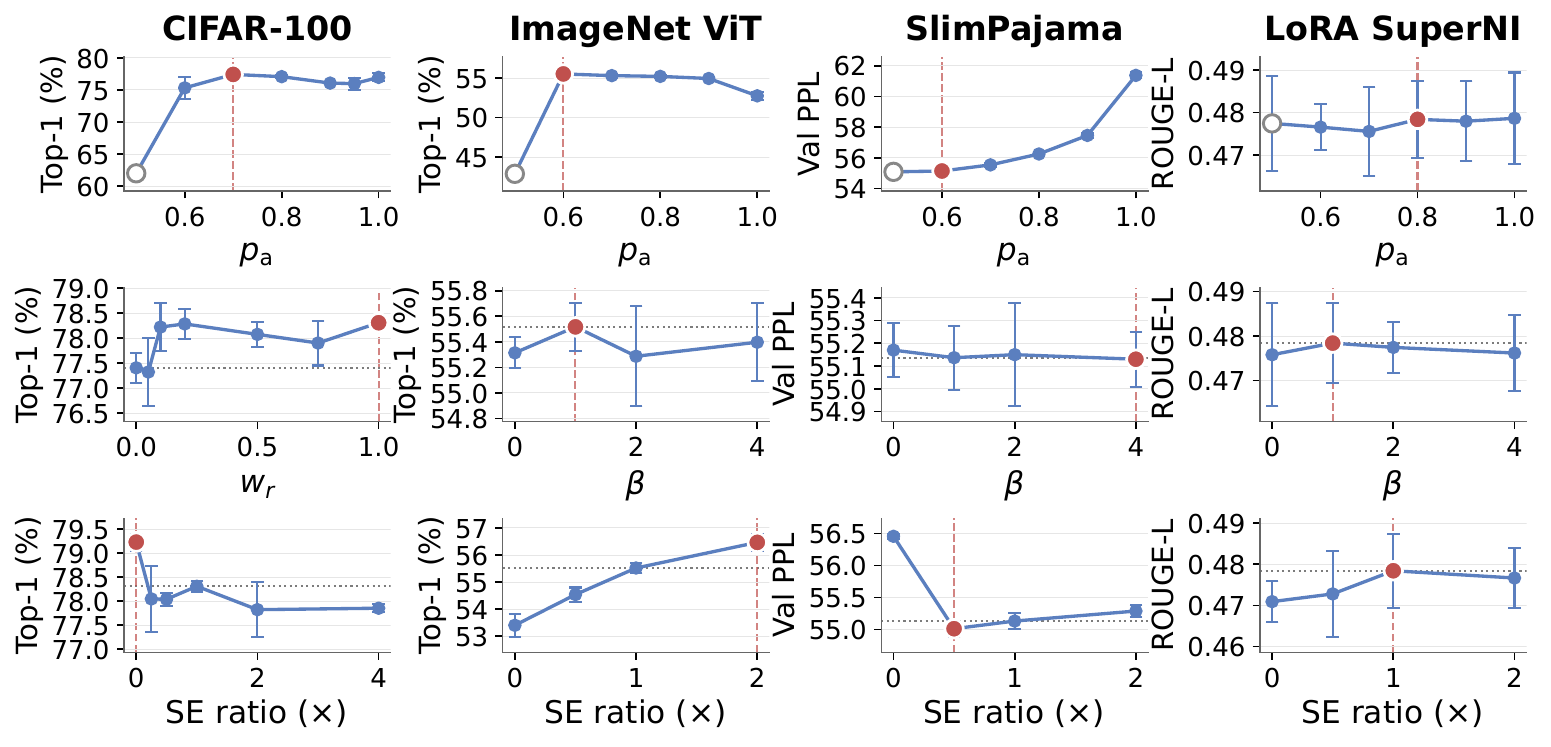}
\caption{\textbf{Hyperparameter ablations across four settings.} Row~1: activation probability $p_{\mathrm{a}}$. Row~2: imbalance-amplification exponent $\beta$ on the imbalanced settings; on balanced CIFAR-100 (equal training images per category), the per-category formula reduces to a scalar for all $\beta$, so we sweep the warmup ratio $w_r$ instead. Row~3: shared-expert capacity ratio $X$, the SE's size relative to one routed branch (FFN hidden width on ImageNet and SlimPajama, LoRA rank on SuperNI, residual blocks per layer group on CIFAR-100, whose $0.25$--$4\times$ are nominal, Tab.~\ref{tab:ablation_combined}). Each panel: mean over 3 seeds, error bars $\pm 1$ sample standard deviation (the Phase~3a--3c and 5a--5c tables of App.~\ref{app:nlp_miniablation}--\ref{app:vit_miniablation} use the population form). The deployed operating point is marked with a vertical line and a filled red marker; hollow markers denote the algebraic mechanism-OFF point ($p_{\mathrm{a}}{=}p_{\mathrm{i}}$, excluded a priori; on SuperNI, $p_{\mathrm{a}}{=}1$, drawn filled, was excluded after the sweep, App.~\ref{app:lora_miniablation}); the dotted horizontal reference line carries the preceding row's value at its selected operating point, i.e., its best point after the exclusions. The ImageNet, SlimPajama, and SuperNI panels are the reduced-data selection sweeps of App.~\ref{app:nlp_miniablation}--\ref{app:lora_miniablation} ($20\%$ ImageNet subset, $100$M tokens, $20\%$ SuperNI training subset), so they score worse than ours in Tabs.~\ref{tab:vit}--\ref{tab:lora} (lower top-1 and ROUGE-L, higher PPL).}
\label{fig:ablation_grid}
\end{figure}

Section~\ref{sec:nlp_miniablation} summarizes the hyperparameter sweeps; Fig.~\ref{fig:ablation_grid} plots them for all four settings (including the CIFAR-100 $w_r$ sweep), and the full per-value CIFAR-100 $p_{\mathrm{a}}$ and SE-ratio sweeps are below.

\begin{table}[h]
\centering
\small
\caption{\textbf{Full Soft SpecDrop hyperparameter ablation on CIFAR-100.} MultiBranchResNet110, $K{=}20$, 200 epochs, 3 seeds.
Phase~A: $p_{\text{active}}$ sweep at $w_r{=}0$ (no warmup) with a $1.0\times$ shared expert. Phase~B: SE capacity ratio at $(p_{\mathrm{a}}, w_r){=}(0.7, 1.0)$ (cosine warmup, deployed); the SE has the routed branches' channel widths and $3/7/18/35/75$ residual blocks per layer group for the nominal $0.25/0.5/1.0/2.0/4.0\times$, against $18$ per routed branch. Bold = phase-wise best. The Phase~B $0\times$ SE value $79.23\pm 0.17$ is the deployed number reported in Tab.~\ref{tab:cifar_main}; the $1.82$ gap from Phase~A's $p_{\mathrm{a}}{=}0.7$ optimum ($77.41$) combines the cosine-warmup contribution ($+0.90$ at the same $1.0\times$ SE, Phase~B) with the effect of removing the SE ($+0.92$).}
\label{tab:ablation_combined}
\begin{tabular*}{\textwidth}{@{\extracolsep{\fill}}l l *{7}{c}@{}}
\toprule
\multicolumn{2}{@{}l}{\textbf{Phase A}: $p_{\text{active}}$} & 0.5 & 0.6 & \textbf{0.7} & 0.8 & 0.9 & 0.95 & 1.0 \\
\cmidrule(lr){3-9}
\multicolumn{2}{@{}l}{Top-1 (\%)} & 62.00 & 75.31 & $\mathbf{77.41}$ & 77.06 & 76.05 & 75.97 & 76.95 \\
\multicolumn{2}{@{}l}{$\pm$ std}  & 0.36  & 1.72  & 0.30             & 0.21  & 0.17  & 0.93  & 0.64 \\
\midrule
\multicolumn{2}{@{}l}{\textbf{Phase B}: SE ratio} & \textbf{0$\times$} & 0.25$\times$ & 0.5$\times$ & 1.0$\times$ & 2.0$\times$ & 4.0$\times$ & --- \\
\cmidrule(lr){3-9}
\multicolumn{2}{@{}l}{Top-1 (\%)} & $\mathbf{79.23}$ & 78.05 & 78.04 & 78.31 & 77.82 & 77.85 & --- \\
\multicolumn{2}{@{}l}{$\pm$ std}  & 0.17             & 0.68  & 0.13  & 0.11  & 0.57  & 0.09  & --- \\
\bottomrule
\end{tabular*}
\end{table}

\paragraph{Extended take-aways.}
Phase~A spans the gradient ratio $p_{\mathrm{a}}/p_{\mathrm{i}}$ of Thm.~\ref{thm:gradient} from $1$ (no concentration) to hard routing: fully uniform routing ($p_{\mathrm{a}}{=}0.5$) collapses to $62.00\%$, and fully hard routing ($p_{\mathrm{a}}{=}1.0$, $76.95\%$) falls below the dual-probability optimum $p_{\mathrm{a}}{=}0.7$ ($77.41\%$), though within about $0.7\sigma$ ($\sigma$ from the two cells' combined sample SDs, as in App.~\ref{app:random_a}) with the shared expert present.
Phase~B reveals a setting-dependent finding: an always-on shared expert helps on the imbalanced settings (SlimPajama and ImageNet clearly, SuperNI weakly; App.~\ref{app:nlp_miniablation}--\ref{app:lora_miniablation}) but dilutes the routing signal on balanced CIFAR-100.

\subsection{Stochastic SpecDrop}

Section~\ref{sec:soft_specdrop} of the main text presents Soft SpecDrop (deterministic soft weighting) as our canonical variant; here we document the original Bernoulli sampling scheme for completeness.
For each training sample with category $c$, we draw an independent Bernoulli mask per module:
\begin{equation}
    m_k \sim \mathrm{Bernoulli}\bigl(p_k(c)\bigr), \qquad \text{output}_{\text{train}} = \frac{\sum_k m_k\, h_k}{S} \;+\; h_{\mathrm{SE}},
\end{equation}
where $S = p_{\text{active}} + (K-1)p_{\text{inactive}}$ is the fixed constant of Prop.~\ref{thm:fixed_denom} and $h_{\mathrm{SE}}$ is the shared-expert output defined in Eq.~\ref{eq:soft_specdrop} (added \emph{after} $\div S$, by the same design as Soft SpecDrop; set to zero when no shared expert is used).
At inference, $m_k = p_k(c)$ deterministically, giving the same output expression as Soft SpecDrop.
The stochastic masking acts as a regularizer analogous to standard dropout. By Prop.~\ref{thm:fixed_denom}(a) a merge is unbiased when its branch outputs do not depend on the mask, i.e., at the first routed layer group; because one mask is reused at all three layer groups and the branches are nonlinear, deeper merges and the network output match the deterministic test pass only approximately, as with standard dropout's weight scaling.
Soft SpecDrop removes Bernoulli mask-sampling variance, so the $p_k(c)/S$ scaling in Thm.~\ref{thm:gradient} is exact at training time (the $p_{\mathrm{a}}/p_{\mathrm{i}}$ ratio still requires the stated base-gradient symmetry); this motivates the choice of Soft as the default variant in the main experiments.

\subsection{Mask \texorpdfstring{$\times$}{x} Denominator Ablation}

We decompose Soft SpecDrop into its orthogonal design axes: whether masks are \emph{stochastic} (Bernoulli $m_k$) or \emph{deterministic} (soft weight $p_k(c)$), and whether the merge denominator is \emph{fixed} ($\div S$) or \emph{stochastic} ($\div \sum_k m_k$).
Proposition~\ref{thm:fixed_denom}(a) shows that the fixed denominator keeps a merge unbiased (as would any mask-independent constant; $S$ is the one whose weights form a convex combination); the naive stochastic denominator adds a Jensen bias (App.~\ref{app:proofs}).
Table~\ref{tab:denom_ablation} isolates the three distinct combinations (for the soft mask $\sum_k m_k = \sum_k p_k(c) = S$, so soft + stochastic denominator coincides with soft + fixed $S$) at full scale (MultiBranchResNet110, $K{=}20$, 200 epochs, 3 seeds), plus two random-dropout references that remove the category-conditioned matrix $\mathbf{A}$ entirely.

\begin{table}[!htbp]
\centering
\small
\caption{\textbf{Mask $\times$ denominator ablation} on CIFAR-100 (MultiBranchResNet110, $K{=}20$, 200 epochs; 3 seeds for all rows, with standard deviation reported only for the deployed Soft + fixed $S$ row). The Soft + fixed $S$ row uses the deployed canonical config $(p_{\mathrm{a}}, w_r, \text{SE}){=}(0.7, 1.0, 0\times)$ matching Tab.~\ref{tab:cifar_main}; the Bernoulli (Stochastic) variants use $(p_{\text{active}}, p_{\text{inactive}}){=}(0.9, 0.1)$ with the same $w_r{=}1.0$ warmup; random-dropout (no-$\mathbf{A}$) comparisons remove category conditioning entirely, keeping every branch with constant probability $0.5$ (no warmup; fixed denominator $S{=}K{\cdot}0.5{=}10$). The decomposition varies deterministic weighting, fixed-denominator, and category conditioning along three design axes; mask probabilities and warmup differ across regimes as listed, so cross-regime comparisons additionally absorb this shift.}
\label{tab:denom_ablation}
\begin{tabular*}{\textwidth}{@{\extracolsep{\fill}}llcc@{}}
\toprule
\textbf{Mask} & \textbf{Denominator} & \textbf{Train--test match} & \textbf{Top-1 (\%)} \\
\midrule
Soft (deterministic $p_k(c)$)           & fixed $S$              & exact (deterministic)                       & $\mathbf{79.23 \pm 0.17}$ \\
Stochastic (Bernoulli $m_k$)            & fixed $S$              & 1st merge only (Prop.~\ref{thm:fixed_denom}) & $64.55$ \\
Stochastic (Bernoulli $m_k$)            & stochastic $\sum_k m_k$ & no (Jensen bias)                            & $67.54$ \\
\midrule
Random dropout (no $\mathbf{A}$)        & fixed $S$              & 1st merge only                              & $58.97$ \\
Random dropout (no $\mathbf{A}$)        & stochastic $\sum_k m_k$ & no (Jensen bias)                            & $59.53$ \\
\bottomrule
\end{tabular*}
\end{table}

\paragraph{Take-aways.}
Two quantitative effects decompose the method.
\emph{(i)~Deterministic soft weighting is the dominant contributor.}
Within category-conditioned variants at fixed $S$, replacing the deterministic $p_k(c)$ with Bernoulli sampling costs $14.68$ points ($79.23 \to 64.55$) --- the single largest effect in the table.
\emph{(ii)~Category conditioning contributes within the Bernoulli regime.}
Comparing the two Bernoulli-masked variants at fixed $S$ \emph{with} vs.\ \emph{without} category conditioning gives $64.55$ vs.\ $58.97$, $+5.58$ points; the no-$\mathbf{A}$ rows also differ in keep probability and warmup (Tab.~\ref{tab:denom_ablation} caption), so this contrast does not isolate $\mathbf{A}$ alone.
The fixed-versus-stochastic-denominator distinction, theoretically predicted to favor fixed via Jensen-bias elimination, is less clean empirically in the Bernoulli regime: Bernoulli with fixed $S$ ($64.55$) vs.\ with $\sum_k m_k$ ($67.54$) (direction inverted from the per-merge prediction, which holds exactly only at the first layer group; consistent with the Bernoulli experiments' own seed-to-seed std ($2$--$3$ points) dominating the bias correction).
In the deterministic regime --- where we operate in Soft SpecDrop --- the fixed denominator $S$ applies by linearity (Prop.~\ref{thm:fixed_denom}(a)) and the Jensen-bias issue does not arise.
The random-dropout floor ($58.97$/$59.53$) confirms that without $\mathbf{A}$, the mask mechanism alone is worse than the $63.08$ No-Routing baseline by $\sim 4$ points: it is the \emph{asymmetric signal shaped by categories}, not the masking itself, that drives the gain.

\subsection{Assignment-Matrix Geometry: Round-Robin vs.\ Random Permutation}
\label{app:random_a}

To rule out the hypothesis that the specific round-robin assignment $\mathbf{A}_{\text{rr}}$ is responsible for the gain, we replace $\mathbf{A}_{\text{rr}}$ with a random permutation $\mathbf{A}_{\text{rand}}$ (each category still assigned to exactly one module, but the category-to-module mapping is a uniform random bijection; assignment seeds $42/123/456$).
Random permutation gives $78.69 \pm 0.33\%$ versus round-robin's $79.23 \pm 0.17\%$ --- a gap of $-0.54$ points (${\approx}1.4\sigma$ under the two configurations' combined seed variance).
Both arms share $(p_{\text{active}}, p_{\text{inactive}}){=}(0.7, 0.3)$, so the comparison varies only the assignment geometry: at $M{=}K$, a random bijection $\mathbf{A}$ performs close to round-robin, i.e., the gain does not depend on a carefully chosen category-to-module mapping.

\section{Baseline Implementation Details}
\label{app:baseline_adaptations}

We describe the implementation of each baseline used in the main comparisons.
The dense CIFAR-100 baselines (Section~\ref{sec:cifar_results}) run on the dense ResNet-110 backbone following their original papers (adaptations below), with no multi-branch wrapping; No-Routing and HardCategory use the MultiBranch backbone.
The NLP MoE baselines (Section~\ref{sec:nlp}) are paper-canonical architectures, each implemented as a dedicated model class; the algorithm-plugin pathway used by SpecDrop is bypassed because routing is intrinsic to these models.

\paragraph{Checkpoint selection.} Each run reports its best evaluation during training on the reported split: CIFAR-100 test top-1 (CIFAR-100 has no validation split), ImageNet-1K validation top-1, SlimPajama validation perplexity, and SuperNI ROUGE-L on the held-out tasks (best of $3$ epochs). For the main-table runs (Tabs.~\ref{tab:cifar_main}--\ref{tab:lora}), the best epoch lies in the last $10\%$ of training on CIFAR-100 and ImageNet, with final-epoch values lower by at most $0.37$ and $0.10$ top-1 (CIFAR-100 final epoch: ours $79.04$, dense $74.33$, No-Routing $62.90$); SlimPajama's best perplexity is the final epoch's in every such run; on SuperNI, final-epoch values are lower by up to $0.010$ ROUGE-L (final-epoch means: ours $0.5068$, No-Routing$+$SE $0.5094$, HydraLoRA $0.5113$), a tie either way.

\subsection{CIFAR-100 Baselines}

\paragraph{Stochastic Depth~\citep{huang2016deep}.}
Block-level Bernoulli drop on dense ResNet-110, with the linear survival schedule $p_\ell = 1 - \frac{\ell-1}{L-1}(1 - p_L)$, $\ell = 1,\dots,L$, and $p_L\!=\!0.5$ (Equation~4 of the original paper, indexed so that the first block is never dropped).
Survival probabilities are stored on the modules and consumed inside each \texttt{BasicBlock.forward}; a dropped block outputs the \texttt{ReLU} of its shortcut, as in Equation~2 of the original paper.
At inference the per-block expected scaling is applied.

\paragraph{Example-Tied Dropout~\citep{maini2023example}.}
Per-example fixed binary masks over a memorization-channel subset, applied to the output of every residual block of the dense ResNet-110 (after the residual addition and \texttt{ReLU}).
We follow the paper's deployment protocol from its Section~6.2: at test time the memorization neurons are zeroed out, leaving only the generalization channels active.
This protocol improves test accuracy on the 9-layer ResNet of the original paper, but on the $54$-block ResNet-110 it drops a quarter of each block's output channels per example during training and zeroes half of them at inference ($p_{\text{gen}}{=}0.5$ here; the original's CIFAR-10 setting uses $p_{\text{gen}}{=}0.4$, $p_{\text{mem}}{=}0.2$); we report the original protocol rather than tuning it.

\paragraph{Contextual Dropout~\citep{fan2021contextual}.}
Gaussian variant with the paper's scaled sigmoid $\sigma_t(\alpha) = \sigma(0.01\alpha)$, placed inside the residual path between the two convolutions, matching the WRN placement of Figure~6 of the paper.
The context network uses reduction ratio $\gamma{=}8$, the low end of the paper's $8$--$16$ range (with at least $4$ hidden units).

\paragraph{No-Routing.}
The same MultiBranchResNet110 architecture used by Soft SpecDrop, with all branch weights fixed to $1/K$ for every input.
This isolates the cost of the multi-branch architecture from the routing signal: any gap between Soft SpecDrop and No-Routing is attributable purely to the category-conditioned routing weights, since both use the same fixed-denominator merge.

\subsection{NLP Baselines}

All NLP baselines share the same backbone (6-layer Transformer LM, hidden $384$, $6$ heads, max sequence length $512$, vocab $50{,}257$) and the same training recipe (Section~\ref{sec:nlp}); only the FFN and routing differ.
Each MoE baseline (Switch, Hash Layers, SMoE-Dropout, DEMix) is implemented as a stand-alone model class so that routing is intrinsic to the architecture rather than a plug-in mask layer; the No-Routing controls reuse our multi-branch model with uniform $1/K$ branch weights.

\paragraph{Switch Transformer~\citep{fedus2022switch}.}
$N\!=\!32$ experts with FFN hidden $48$ each (expert widths sum to the dense FFN hidden $1536$; total parameters $30.288$M vs.\ dense $30.143$M, Tab.~\ref{tab:nlp}).
Each token is routed to the top-$1$ expert via a learned linear router; the selected expert's output is scaled by its softmax gate value $p_t[i^*]$ (\citealp{fedus2022switch}, \S 2.1), so gradients flow through the gate naturally.
The auxiliary load-balance loss (paper Eq.~4) is added with weight $0.01$.

\paragraph{Hash Layers~\citep{roller2021hash}.}
Scaled-down $N\!=\!8$ experts with FFN hidden $192$ each (parameter-matched; the original paper's smallest configuration is $N{=}16$).
Each token is mapped to one expert by a fixed random hash table over the vocabulary; the hash table is drawn once at init from a fixed seed and frozen.
No routing parameters and no auxiliary loss; per-layer hash tables use offset seeds.

\paragraph{SMoE-Dropout~\citep{chen2023smoe}.}
$N\!=\!16$ experts (the largest width in \citeauthor{chen2023smoe}'s Figure~5(b)) with FFN hidden $96$ each.
Routing uses a fixed random Linear projection (no learned parameters); top-$k$ selection with softmax-normalized weights. \citeauthor{chen2023smoe} grow $k$ linearly from $k_{\text{init}}$ to $N$ over training ($k_t = k_{\text{init}} + (N - k_{\text{init}}) \cdot t / T$); in our runs this per-epoch update was skipped (its dispatch check does not fire under \texttt{torch.compile}), so all runs train with $k{=}k_{\text{init}}{=}1$ and evaluate with all $N{=}16$ experts.
We disable the per-expert Bernoulli dropout (\texttt{expert\_drop\_prob}=$0.0$) to match the paper.

\paragraph{DEMix~\citep{gururangan2022demix}.}
One FFN per domain ($N\!=\!7$), FFN hidden $220$ each (parameter-matched).
Hard per-domain routing using the document's domain tag during training; at inference the document's domain tag also selects its expert (oracle-domain evaluation; the mixture-of-experts inference of \citet{gururangan2022demix}, \S 5.1, was not used).

\paragraph{Dense Transformer.}
A single FFN of hidden $1536$, no routing.
This is the parameter and compute reference for every multi-branch method.

\paragraph{No-Routing (NLP).}
$K\!=\!7$ FFN branches of width $220$ each, combined by uniform $1/K$ weighting at every layer.
Architecturally identical to the multi-branch SpecDrop variant but with the routing signal removed; serves as the mechanism-OFF reference at this branch count and width without a shared expert (the matched reference for our deployed shared-expert configuration is No-Routing$+$SE, Tab.~\ref{tab:nlp}).

\paragraph{Batch-size disclosure (NLP).}
All NLP perplexity numbers (Sec.~\ref{sec:nlp}, Table~\ref{tab:nlp}) are from training runs with per-device batch size $32$ sequences of length $512$ ($16{,}384$ tokens per optimizer step). At a doubled per-device batch ($64$), ours and the matched-SE control remain tied ($40.74$ vs.\ $40.70$ PPL, seed $42$).

\paragraph{NLP training-regime disclosure (multi-epoch over Chinchilla-optimal).}
\label{app:nlp_regime_disclosure}
The main NLP experiments (Sec.~\ref{sec:nlp}) train for $10$ epochs on $500$M unique SlimPajama tokens (the mini-ablation, Sec.~\ref{sec:nlp_miniablation}, and the BGE clustering of App.~\ref{app:nlp_clusters} use a $100$M-token subset; the per-chunk $\Delta\text{CE}$ there uses the main $500$M-token runs), totaling $5$B token-passes ($\approx 167$ tokens per parameter at $30$M scale, well above the Chinchilla-optimal $\sim 20$ tokens per parameter); the $125$M scale-up (App.~\ref{app:scaling_check}) preserves this regime, changing model size, halving the per-device batch to $16$, and using one $(p_{\mathrm{a}}, p_{\mathrm{i}})$ for all domains.
We retain the $10$-epoch regime to preserve internal consistency across the NLP-side ablation cells (Sec.~\ref{sec:nlp_miniablation}, App.~\ref{app:nlp_miniablation}); a $1$-epoch sanity verification on the same $500$M-token cache is reported in App.~\ref{app:nlp_regime_sanity}.
\textbf{All NLP baselines (Dense, No-Routing, No-Routing+SE, Switch, Hash Layers, DEMix, SMoE-Dropout, ours) use the identical $10$-epoch regime}, so cross-method rankings of Table~\ref{tab:nlp} are internally fair.
We acknowledge this regime is over-trained relative to standard $1$-epoch LM-pretraining practice (Switch, GPT-3, LLaMA, Chinchilla); a $30$M $\times$ $1$-epoch sanity verification of the ours-vs-matched-SE tie is reported in App.~\ref{app:nlp_regime_sanity}, and the $125$M scale-up of App.~\ref{app:scaling_check} shows the same direction (ours slightly behind) at $4\times$ scale.
The granularity-alignment thesis itself (Sec.~\ref{sec:discussion}) is anchored to data-modality properties --- intra-chunk heterogeneity ($56.1\%$ of $512$-token chunks span $\geq 2$ BGE clusters at $k{=}7$ over the full $9{,}766$-chunk validation set, App.~\ref{app:nlp_clusters}) and silhouette modality asymmetry (BGE $s_{\max}{=}0.031$ vs DINOv2 CIFAR $0.069$) --- which are properties of the data and embeddings, independent of training-epoch count.

\paragraph{Hyperparameter selection.}
Our method's operating point is identified via a three-stage sequential ablation (App.~\ref{app:ablation_full} for CIFAR, App.~\ref{app:vit_miniablation} for ImageNet, App.~\ref{app:nlp_miniablation} for NLP, App.~\ref{app:lora_miniablation} for SuperNI) at $3$ seeds per cell. Baselines use the paper-canonical hyperparameters specified in their original publications (Switch load-balance weight $0.01$, Hash Layers frozen random seeds, SMoE-Dropout \texttt{expert\_drop\_prob}${=}0$, DEMix hard domain routing, Mod-Squad \texttt{mi\_weight}${=}0.001$), except that COMET keeps $25\%$ of each MLP's hidden units (\texttt{p\_keep}${=}0.25$; \citet{shaier2025comet} use $0.5$ for ViT MLPs); other material deviations from the original configurations are stated per baseline (above and App.~\ref{app:vit_miniablation}, \ref{app:lora_miniablation}).

\subsection{ImageNet BREEDS-46 Construction Algorithm}
\label{app:breeds_construction}

The $M{=}46$ supercategory partition is derived from the BREEDS~\citep{santurkar2021breeds} curated WordNet hierarchy via a recursive expansion controlled by two parameters, $T$ (max-leaves before forced expansion) and $C$ (max-children allowed for an expansion to be accepted). Starting from BREEDS' level-3 nodes ($29$ groups covering most of the $1000$ ImageNet-1K fine classes), for each node whose number of ImageNet leaves exceeds $T$ we attempt to expand into its WordNet children; the expansion is accepted only when the node has at most $C$ non-empty child groups, which prevents degenerate splits where one child holds nearly all leaves and the rest become singletons (e.g., \emph{carnivore} has $\sim 25$ children but only \emph{dog} carries substantial mass).

With $T{=}60$, $C{=}10$ as our depth/branching cutoffs, this recursion produces $45$ groups covering $890$ classes; the remaining $110$ ImageNet classes that do not appear in the BREEDS curated tree are pooled into a single \emph{miscellaneous} group, yielding $M{=}46$. The resulting partition is imbalanced: the largest group is \emph{carnivore} ($158$ classes), followed by \emph{miscellaneous} ($110$), \emph{man-made structure} ($62$), \emph{bird} ($59$), and \emph{equipment} ($49$); a long tail of $21$ groups holds only $2$--$8$ classes each (e.g., \emph{aquatic mammal}, \emph{marsupial}, \emph{vascular plant}). These $T$, $C$ values are our hyperparameters for the recursive expansion (the BREEDS curated tree itself comes from \citet{santurkar2021breeds}); at fixed $C{=}10$, lowering the cutoff to $T{=}20$ refines the partition to $M{=}73$, while raising it to $T{=}100$ leaves $M{=}46$ unchanged (the only remaining tree group above that cutoff, \emph{carnivore}, does not admit an expansion into ${\leq}C$ non-empty subgroups).

\subsection{SuperNI \texorpdfstring{$K{=}20$}{K=20} Clustering Construction}
\label{app:superni_clusters}

The $M{=}K{=}20$ task clusters used in Section~\ref{sec:lora} are constructed by frequency-cutoff over the SuperNI~\citep{wang2022supernaturalinstructions} \emph{Domains} field, not by $K$-means or sentence-embedding clustering.

\textbf{Algorithm.} For each task in the $756$-task English training split: (i) take its first \emph{Domains} entry and normalize it by extracting the root segment of the hierarchical path (e.g., \emph{``Commonsense $\to$ Concepts and Relations $\to$ Social Commonsense''} becomes \emph{``Commonsense''}), yielding $\sim 72$ unique normalized root domains across the train split; (ii) count root-domain frequency over training tasks; (iii) retain the top-$19$ most frequent root domains as cluster IDs $0$--$18$, with all remaining tasks assigned to cluster ID $19$ (\emph{miscellaneous}). At test time, the $119$ held-out tasks map to clusters via the same normalize-and-lookup function; tasks whose normalized root domain is unseen in the train-frequency table fall back to the miscellaneous cluster.

This metadata-based construction is chosen for determinism and for consistency with our ImageNet partition, which likewise groups classes by existing metadata and pools the remainder into a \emph{miscellaneous} group (App.~\ref{app:breeds_construction}). We did not perform $K$-means over the task definitions of \citet{wang2022supernaturalinstructions}; that is a reasonable alternative left to future work. The choice $K{=}20$ matches the CIFAR-100 superclass count for cross-setting comparison, and the trailing miscellaneous cluster holds $30$ of the $119$ test tasks, large enough to avoid being dominated by a single task family.

\section{Proofs}
\label{app:proofs}

\subsection{Proof of Theorem~\ref{thm:gradient} (Gradient Concentration)}

\begin{proof}
Let $z = \sum_k m_k h_k / S$ denote the merged routed output, so $\hat{y} = f_{\text{head}}(z)$.
By the chain rule, $\frac{\partial \mathcal{L}}{\partial \theta_k} = \frac{\partial \mathcal{L}}{\partial z} \cdot \frac{m_k}{S} \cdot \frac{\partial h_k}{\partial \theta_k}$, so $\|\partial\mathcal{L}/\partial\theta_k\| = (m_k/S)\,G_k$ with the \emph{base gradient} $G_k = \|\frac{\partial \mathcal{L}}{\partial z}\cdot\frac{\partial h_k}{\partial \theta_k}\|$.
Taking expectations over $(x, m)$ given $c$, assumption (i) ($G_k$ independent of $m_k$ given $x$) gives $\mathbb{E}[m_k G_k \mid x] = p_k(c)\,\mathbb{E}[G_k \mid x]$, hence
\[
\mathbb{E}\!\left[\!\left\|\frac{\partial \mathcal{L}}{\partial \theta_k}\right\| \;\middle|\; c\right] = \frac{p_k(c)}{S} \cdot \mathbb{E}_{x|c}\!\left[G_k\right].
\]
For assigned $c$ ($A_{ck}{=}1$), $p_k(c)=p_{\text{active}}$; for unassigned $c'$ ($A_{c'k}{=}0$), $p_k(c')=p_{\text{inactive}}$.
Under assumption (ii) the base gradient expectations $\mathbb{E}_{x|c}[G_k]$ and $\mathbb{E}_{x|c'}[G_k]$ coincide, and the ratio reduces to $p_{\text{active}}/p_{\text{inactive}}$.
\end{proof}

The independence assumption in Theorem~\ref{thm:gradient} is a first-order approximation: in practice, the mask realization $m_k$ affects the merged output and thus the loss landscape, creating higher-order dependencies. The equal-base-gradient assumption (ii) is likewise an idealization: once specialization develops, base gradients differ across categories.
However, the linear scaling $p_k(c)/S$ dominates the gradient expectation, and our empirical proxy is consistent with this first-order prediction: the diagonal-to-off-diagonal pruning-sensitivity ratio (a post-training accuracy-drop proxy) measured at $(p_{\text{active}}, p_{\text{inactive}}) = (0.7, 0.3)$ is $1.87\times$ (seed $42$, Fig.~\ref{fig:heatmap}), close to the theoretical $p_{\text{active}}/p_{\text{inactive}} = 2.33\times$; since it is measured on the deterministic Soft variant, where the $p_k(c)/S$ scaling is exact, the shortfall reflects assumption (ii), the $w_r{=}1.0$ warmup (under which the training-time ratio $p_{\text{active}}(t)/p_{\text{inactive}}(t)$ rises from $1$ to $2.33$ over training), and the proxy gap (gradient-flow $\to$ end-of-training pruning sensitivity is heuristic, not a tight bridge).

\subsection{Proof of Proposition~\ref{thm:fixed_denom} (Fixed vs.\ Stochastic Denominator)}

\begin{proof}
\textbf{Part (a).}
By linearity of expectation, for any constant $\alpha>0$ we have $\mathbb{E}_m[\sum_k m_k h_k / \alpha] = \sum_k p_k(c) h_k / \alpha$.
Setting $\alpha = S$ matches the test-time forward pass $\sum_k p_k(c) h_k / S$ exactly.
Under round-robin assignment, exactly one module per category has $p_k = p_{\text{active}}$ and the other $K-1$ have $p_k = p_{\text{inactive}}$, so $S = \sum_k p_k(c) = p_{\text{active}} + (K-1)p_{\text{inactive}}$ is the same scalar for every category $c$.
Finally, $\sum_k (p_k(c)/S) = 1$ for all $c$, so the merge is a proper convex combination.

\textbf{Part (b).}
Write $N{=}\sum_k m_k$ and $U{=}\sum_k m_k h_k$; with the convention that the output is $\mathbf{0}$ when $N{=}0$, the naive stochastic merge is $U\,\mathbf{1}[N{>}0]/N$, whose expectation is $\sum_k w_k h_k$ with effective weights $w_k = \mathbb{E}[m_k\,\mathbf{1}[N{>}0]/N]$.
These sum to $\sum_k w_k = P(N{>}0) = 1 - \prod_k (1{-}p_k(c)) < 1$ whenever $p_{\text{active}} < 1$, whereas the test-time weights $p_k(c)/S$ sum to $1$; hence $w_k \neq p_k(c)/S$ for some $k$, and the expectation departs from the test-time forward pass for generic $\{h_k\}$ (e.g., $h_k \equiv h \neq \mathbf{0}$ gives $P(N{>}0)\,h$ instead of $h$).
If instead $p_{\text{active}} = 1$ and $0 < p_{\text{inactive}} < 1$, the assigned branch's weight is $\mathbb{E}[1/N] > 1/S$ by strict Jensen. (At $p_{\text{active}}{=}1$, $p_{\text{inactive}}{=}0$ the mask is deterministic and the two merges coincide; Part (b) concerns non-degenerate masks.)
Heuristically ignoring the $N{=}0$ event, a second-order Taylor expansion of $1/N$ around $\mu_N = \mathbb{E}[\sum_k m_k] = S$ gives $\mathbb{E}[1/N] \approx 1/S + \mathrm{Var}(N)/S^3$, so the relative magnitude of the Jensen bias scales as $\mathrm{Var}(\sum_k m_k)/S^2$.
\end{proof}

\begin{remark}[Quantifying the Jensen Bias]
With $p_{\textup{active}} = 0.9$, $p_{\textup{inactive}} = 0.1$, $K = 4$: $S = 1.2$, $\textup{Var}(N) = 0.09 + 3 \times 0.09 = 0.36$.
The second-order Taylor scale $\mathrm{Var}(N)/S^2{=}0.25$ is unreliable here ($\sigma_N/\mu_N{=}0.5$ violates the small-deviation regime); we therefore report the explicit per-branch bias instead.
Accounting for the $P(N{=}0)=0.1\times 0.9^3{\approx}0.0729$ all-zero mask (where the naive merge outputs $0$ and the unconditional weights $\mathbb{E}[m_k/N]$ sum to $P(N{>}0)\approx 0.9271$ rather than $1$): the unassigned branches' effective weight shifts from $0.083$ to $\approx 0.0511$, and the assigned branch's effective weight shifts from the intended $0.750$ to $0.9271 - 3{\times}0.0511 \approx 0.774$, a $+3.2\%$ relative bias for the assigned branch and $-38.7\%$ for unassigned branches.
This Jensen bias governs only the Stochastic SpecDrop variant of App.~\ref{app:stoch_soft}; for the deployed Soft variant the deterministic weights $p_k(c)$ remove the stochastic-denominator issue entirely and Prop.~\ref{thm:fixed_denom}'s empirical role becomes the magnitude calibration of $S$ in Eq.~\ref{eq:soft_specdrop} (Sec.~\ref{sec:soft_specdrop}).
In the Bernoulli-regime ablation of Tab.~\ref{tab:denom_ablation} (Bernoulli with fixed $S$: $64.55$; with $\sum_k m_k$: $67.54$), the Bernoulli seed-to-seed std ($2$--$3$ points) dominates the Jensen-bias correction at this scale.
\end{remark}

\subsection{Per-Category $S$-Invariance (Extension to Eq.~\ref{eq:per_category})}
\label{app:per_cat_proof}

The per-category schedule of Sec.~\ref{sec:per_category} generalizes $(p_{\mathrm{a}}, p_{\mathrm{i}})$ to $(p_{\mathrm{a}}^c, p_{\mathrm{i}}^c)$ while preserving $S^c{=}S$ exactly.

\begin{lemma}[Per-Category $S$-Invariance]
\label{lem:s_invariance}
Let $\mathrm{gap}_c{=}(p_{\mathrm{a}}{-}p_{\mathrm{i}})\,[(1{-}\pi_c)/(1{-}1/M)]^{\beta}$ with $\beta{\geq}0$, $\{\pi_c\}$ any probability distribution on $\mathcal{C}$, and $(p_{\mathrm{a}}^c, p_{\mathrm{i}}^c){=}(S/K{+}\mathrm{gap}_c(K{-}1)/K,\, S/K{-}\mathrm{gap}_c/K)$ per Eq.~\ref{eq:per_category}.
Then $S^c \,\triangleq\, p_{\mathrm{a}}^c + (K{-}1)\, p_{\mathrm{i}}^c \,=\, S$ for every $c$, $\beta$, and $\{\pi_c\}$.
\end{lemma}

\begin{proof}
By direct expansion,
\[
S^c = \tfrac{S}{K} + \mathrm{gap}_c\tfrac{K{-}1}{K} + (K{-}1)\bigl[\tfrac{S}{K} - \tfrac{\mathrm{gap}_c}{K}\bigr] = \tfrac{S}{K}\!\cdot\! K + \mathrm{gap}_c\tfrac{K{-}1}{K} - \mathrm{gap}_c\tfrac{K{-}1}{K} = S. \qedhere
\]
\end{proof}

\begin{corollary}[Prop.~\ref{thm:fixed_denom} extends verbatim to per-category]
Replacing $p_k(c)$ in Eq.~\ref{eq:fixed_merge} with the per-category weights
\[
p_k(c)=A_{ck}\,p_{\mathrm{a}}^c + (1-A_{ck})\,p_{\mathrm{i}}^c
\]
preserves both parts of Prop.~\ref{thm:fixed_denom}: (a) $\mathbb{E}_m[\sum_k m_k h_k/S]{=}\sum_k p_k(c)h_k/S$ remains exact by linearity (as for any constant), and $S^c{=}S$ keeps $(p_k(c)/S)_k$ a convex combination for every $c$ with one category-independent constant; (b) the naive-stochastic denominator still incurs the Jensen bias, now with Bernoulli variance $\mathrm{Var}_c(N){=}p_{\mathrm{a}}^c(1{-}p_{\mathrm{a}}^c) + (K{-}1)p_{\mathrm{i}}^c(1{-}p_{\mathrm{i}}^c)$ that becomes per-category.
The gradient concentration ratio of Thm.~\ref{thm:gradient} generalizes to $p_{\mathrm{a}}^c / p_{\mathrm{i}}^{c'}$ for assigned $c$ and unassigned $c'$; the argmax-branch quantity $I(Z_{\textup{argmax}};C)$ (remark following Thm.~\ref{thm:mi}) is unaffected, since the argmax depends only on $\mathbf{A}$, not on $(p_{\mathrm{a}}^c, p_{\mathrm{i}}^c)$.
\end{corollary}

\begin{remark}[Bounds for Validity]
Validity, $(p_{\mathrm{a}}^c, p_{\mathrm{i}}^c) \in [0,1]^2$, requires
\[
\mathrm{gap}_c \leq \min\bigl(S,\ (K{-}S)/(K{-}1)\bigr).
\]
For our deployed NLP setting, $\max_c \mathrm{gap}_c \leq (p_{\mathrm{a}}{-}p_{\mathrm{i}})(M/(M{-}1))^{\beta}{=}0.2\cdot(7/6)^4{\approx}0.371$ ($K{=}M{=}7$, $p_{\mathrm{a}}{=}0.6$, $p_{\mathrm{i}}{=}0.4$, $S{=}3.0$, $\beta{=}4$; the deployed maximum over the $7$ domains is ${\approx}0.32$), which lies below $\min(3.0,\,(K{-}S)/(K{-}1)){=}\min(3.0,\,0.667){=}0.667$ by a $0.296$ margin; at load the runtime clamps $(p_{\mathrm{a}}^c, p_{\mathrm{i}}^c)$ to $[0,1]$ and raises an error if the clamp breaks $S^c{=}S$.
\end{remark}

\subsection{Routing-Indicator Mutual Information (Theorem~\ref{thm:mi}) and Proof}

\begin{theorem}[Routing-Indicator Mutual Information]
\label{thm:mi}
Let $Z_k \in \{0,1\}$ be the activation indicator for module $k$ under round-robin assignment with $K\,|\,M$ (i.e., $K$ divides $M$) and $C$ uniform on $\{1,\ldots,M\}$.
Then $I(Z_k; C) = \tfrac{1}{K} D_{\textup{KL}}(p_{\textup{active}} \| \bar p) + \tfrac{K-1}{K} D_{\textup{KL}}(p_{\textup{inactive}} \| \bar p)$ exactly, with $\bar p = (p_{\textup{active}} + (K-1)p_{\textup{inactive}})/K$; for $K{\geq}2$ this quantity vanishes iff $p_{\textup{active}} = p_{\textup{inactive}}$ (random dropout).
\end{theorem}

\begin{proof}
We use the assumptions of the theorem: round-robin assignment with $K\,|\,M$ (so each module is assigned to exactly $M/K$ categories) and uniform $C$.
By definition, $I(Z_k; C) = H(Z_k) - H(Z_k | C)$.
Under these two assumptions, exactly $1/K$ of categories assign $p_{\text{active}}$ to module $k$ and $(K{-}1)/K$ assign $p_{\text{inactive}}$, so the conditional entropy factorizes:
$H(Z_k | C) = \frac{1}{K} H_b(p_{\text{active}}) + \frac{K-1}{K} H_b(p_{\text{inactive}})$,
where $H_b(p) = -p\log p - (1-p)\log(1-p)$ is the binary entropy.
The marginal $P(Z_k = 1) = \bar{p}$, so $H(Z_k) = H_b(\bar{p})$.
Thus $I(Z_k; C) = H_b(\bar{p}) - \frac{1}{K}H_b(p_{\text{active}}) - \frac{K-1}{K}H_b(p_{\text{inactive}})$.
Applying the identity $H_b(\bar{p}) - H_b(p) = D_{\text{KL}}(p \| \bar{p}) + (\bar{p} - p)\log\frac{1-\bar{p}}{\bar{p}}$ to each $p \in \{p_{\text{active}}, p_{\text{inactive}}\}$ and combining,
$I(Z_k; C) = \tfrac{1}{K} D_{\text{KL}}(p_{\text{active}}\|\bar p) + \tfrac{K-1}{K} D_{\text{KL}}(p_{\text{inactive}}\|\bar p) + \log\!\frac{1-\bar p}{\bar p}\cdot\!\bigl[\tfrac{1}{K}(\bar p - p_{\text{active}}) + \tfrac{K-1}{K}(\bar p - p_{\text{inactive}})\bigr].$
The bracketed remainder equals $\bar p - [\tfrac{1}{K} p_{\text{active}} + \tfrac{K-1}{K} p_{\text{inactive}}] = 0$ by the definition of $\bar p$, yielding the \emph{exact} equality (no convexity bound needed).
When $p_{\text{active}} = p_{\text{inactive}}$, $\bar{p} = p_{\text{active}}$ and all KL terms vanish; conversely, for $K{\geq}2$ and $p_{\text{active}} \neq p_{\text{inactive}}$, $\bar p$ lies strictly between them, so both KL terms are positive.
\end{proof}

\begin{remark}[Numerical Evaluation]
At the Bernoulli configuration $(p_{\textup{active}}, p_{\textup{inactive}}) = (0.9, 0.1)$ and $K = 20$ (Stochastic variant of App.~\ref{app:stoch_soft}, $\bar{p} = 0.14$): $I(Z_k; C) \approx \frac{1}{20}(1.4595) + \frac{19}{20}(0.0073) \approx 0.080$ nats. At the deployed Soft variant configuration $(0.7, 0.3)$ ($\bar{p} = 0.32$): $I(Z_k; C) \approx \frac{1}{20}(0.302) + \frac{19}{20}(0.0009) \approx 0.016$ nats --- a smaller per-module signal at the lower $p_{\mathrm{a}}/p_{\mathrm{i}} = 2.33\times$ ratio. Both quantities apply to the binary Bernoulli indicator $Z_k$ and are upper-bounded by $\log 2 \approx 0.693$ nats.
In the Soft SpecDrop variant that we deploy in the main experiments, the soft weight $p_k(c)$ is used deterministically and the per-module Bernoulli quantity is no longer the natural specialization signal; the categorical \emph{argmax-branch} variable $Z_{\textup{argmax}}{:=}\arg\max_k p_k(C)\in\{1,\ldots,K\}$ is deterministic in $C$ under round-robin, giving $I(Z_{\textup{argmax}}; C) = \log K \approx 2.996$ nats exactly at $K{=}20$ (Tab.~\ref{tab:mi_comparison}).
The Bernoulli per-module result and the categorical argmax result measure different quantities; substantive branch-parameter specialization in the deployed regime is therefore measured at the parameter level by pruning sensitivity, giving a $1.87\times$ diagonal-to-off-diagonal ratio at $(p_{\textup{active}}, p_{\textup{inactive}}) = (0.7, 0.3)$ (seed $42$).
\end{remark}

\section{Extended Related Work}
\label{app:related}

\paragraph{Multi-branch architectures and task-specific masks.}
PathNet~\citep{fernando2017pathnet} uses evolutionary selection of module pathways---described as ``evolutionary dropout''---for continual learning.
Piggyback~\citep{mallya2018piggyback} and Supermasks-in-Superposition~\citep{wortsman2020supsup} learn task-specific binary masks over shared weights for task adaptation (incremental and continual settings respectively), storing one binary mask per task rather than separate weights.
MMoE~\citep{ma2018mmoe} uses task-specific learned gates over shared experts for multi-task learning, forming a natural comparison axis with SpecDrop: MMoE learns \emph{soft, input-dependent weights} over modules, one gating network per task, end-to-end (adding trainable parameters per task), while SpecDrop uses a \emph{fixed binary assignment} with a deterministic soft-weight readout (zero extra parameters, no auxiliary losses).
All these mask-based methods learn routing \emph{after} or \emph{during} training for task adaptation, whereas SpecDrop uses a \emph{fixed, predetermined} assignment matrix \emph{during training} to shape what each module learns from scratch.

\paragraph{Theoretical foundations.}
\citet{bena2025modularity} recently showed that structural modularity alone does not guarantee functional specialization without appropriate constraints---a finding that directly motivates SpecDrop's explicit specialization pressure through category-conditioned routing.
Our theoretical analysis (Section~\ref{sec:theory}) builds on this insight, characterizing how category conditioning with nonzero cross-category gradient flow ($p_{\text{inactive}} > 0$) concentrates gradient on assigned branches, which neither structural modularity nor random dropout does on its own.
The interaction between dropout and optimization dynamics~\citep{liu2023dropout} further suggests that stochastic masking can either help or hinder depending on timing and magnitude, consistent with our finding that deterministic category conditioning outperforms stochastic variants.

\section{Additional Experimental Results}
\label{app:additional}

\subsection{CIFAR Pruning-Sensitivity Heatmap and Routing-Level MI (sanity check)}
\label{app:mi_routing}

Figure~\ref{fig:specialization} reports the CIFAR diag-argmax count; Figure~\ref{fig:heatmap} below shows the full heatmap and the diagonal/off-diagonal ratio (seed $42$; three-seed mean $2.02\pm0.17$).

\begin{figure}[h]
\centering
\includegraphics[width=\linewidth]{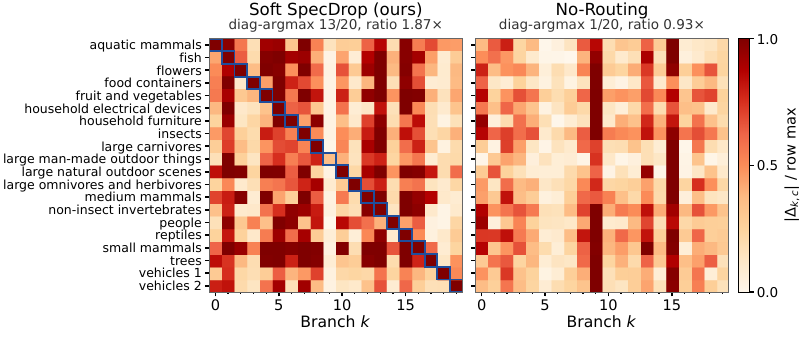}
\caption{\textbf{Pruning sensitivity on CIFAR-100 ($K{=}20$, seed $42$).} Rows: superclasses; cols: branches; cell darkness is $|\Delta_{k,c}|$ (Sec.~\ref{sec:specialization}) normalized to each superclass's maximum; blue outlines (left panel) mark each superclass's round-robin assigned branch; ratio is mean $|\Delta|$ on vs.\ off the diagonal. \emph{Left}: Soft SpecDrop (diag-argmax $13/20$, ties to the lowest branch index as in Tab.~\ref{tab:cifar_main}; ratio $1.87\times$). \emph{Right}: No-Routing (diag-argmax $1/20$, ratio $0.93\times$). Same backbone --- the structural difference is attributable to routing alone.}
\label{fig:heatmap}
\end{figure}

\paragraph{Ties in CIFAR-100 Align.} Zero-ablating a branch removes it from every layer group; for Soft SpecDrop this drives a superclass's accuracy to exactly zero for the assigned branch in $17$--$19$ of $20$ superclasses, but also for $13$--$29$ off-diagonal pairs, so the most-sensitive branch is tied for $11/7/6$ superclasses (seeds $42/123/456$). As in the ViT and NLP Align columns (whose Soft SpecDrop rows have no ties) and the LoRA one (whose Soft SpecDrop rows tie for $8/3/10$ of $15$ clusters, up to floating-point rounding), Tab.~\ref{tab:cifar_main} breaks ties by lowest branch index ($13/15/13$, $68.3\pm5.8$); splitting tied credit equally gives $68.2\pm9.2$, counting unique maxima only $55.0\pm13.2$, and counting any tied maximum $93.3\pm2.9$. No-Routing and HardCategory are unaffected, and the diagonal/off-diagonal ratio does not depend on tie handling. Methods without a category-to-branch assignment (every multi-branch baseline except HardCategory and DEMix) are scored against the index-matched branch $k{=}c$ (categories $0$-indexed), so when $K$ is below the number of categories, categories with $c{\geq}K$ cannot count as hits.

For completeness we also report the categorical argmax-branch quantity $I(Z_{\textup{argmax}}; C)$ where $Z_{\textup{argmax}}{=}\arg\max_k p_k(C)\in\{1,\ldots,K\}$: for Soft SpecDrop this is the branch with the largest $p_k(c)$, and for No-Routing it is undefined by symmetry (reported as the uniform argmax, which collapses to a single branch).
The Soft SpecDrop routing schedule is a deterministic one-to-one map $c \mapsto k(c)$ at $K{=}20$, so $I(Z_{\textup{argmax}}; C) = \log K$ is reached by construction; No-Routing's uniform weights give $I(Z_{\textup{argmax}}; C) = 0$ by symmetry.
Table~\ref{tab:mi_comparison} confirms both values empirically to machine precision.
This contrast is a sanity check on the routing layer and does \emph{not} by itself imply that the branch \emph{parameters} have specialized --- that is what the pruning-sensitivity test of Sec.~\ref{sec:specialization} measures.

\begin{table}[h]
\centering
\small
\caption{\textbf{Routing-level mutual information} on CIFAR-100 (MultiBranchResNet110, $K{=}20$, 200 epochs, 3 seeds 42/123/456). $I(Z_{\textup{argmax}}; C)$ and the argmax counts are exact by construction (no seed variance); Top-1 is mean $\pm$ sample standard deviation.}
\label{tab:mi_comparison}
\begin{tabular*}{\textwidth}{@{\extracolsep{\fill}}lccc@{}}
\toprule
\textbf{Method} & \textbf{$I(Z_{\textup{argmax}}; C)$ (nats)} & \textbf{Unique argmax branches} & \textbf{Top-1 (\%)} \\
\midrule
No-Routing (equal weights)        & $0.0000$                         & $1/20$          & $63.08 \pm 0.04$ \\
\textbf{Soft SpecDrop (ours)}     & $\mathbf{2.9957}$ ($=\log 20$) & $\mathbf{20/20}$ & $\mathbf{79.23 \pm 0.17}$ \\
\bottomrule
\end{tabular*}
\end{table}

\subsection{CIFAR Fine-Label Oracle ($K{=}20$ with fine-label routing)}
\label{app:cifar_fine_label}

To bound the upper endpoint of the partition-alignment spectrum (cf.\ Sec.~\ref{sec:discussion}), we run Soft SpecDrop routed by the fine label instead of the superclass on CIFAR-100: $M{=}100$ fine labels, $K{=}20$ branches, round-robin so each branch covers $5$ fine labels.
This is \emph{deliberately leaky}: the routing key at every site is the fine label mod $20$, which narrows the answer to $5$ of the $100$ fine classes; unlike the $5$ related classes of a superclass, these (labels $c, c{+}20, \dots, c{+}80$) mostly belong to different superclasses.

\begin{center}\small
\begin{tabular*}{\textwidth}{@{\extracolsep{\fill}}lcc@{}}
\toprule
\textbf{Setting} & \textbf{Top-1 ($\pm \sigma$)} & \textbf{vs deployed ours} \\
\midrule
Fine-label oracle ($M{=}100$, $K{=}20$, $5$ labels/branch) & $\mathbf{93.35 \pm 0.23}$ & $+14.12$ \\
Deployed ours ($M{=}20$, $K{=}20$, superclass routing) & $79.23 \pm 0.17$ & --- \\
ResNet-110 dense reference & $74.48$ & $-4.75$ \\
\bottomrule
\end{tabular*}
\end{center}

\paragraph{This is not a deployable comparison.}
The $+14.12$ top-1 lift reflects \emph{oracle-routing}: at training and inference time the routing key narrows the answer to $5$ mostly unrelated fine classes.
We report it as the empirical upper-bound endpoint of the partition-alignment spectrum --- when the routing key nearly determines the answer, accuracy approaches the oracle's per-class capacity ($\sim 93\%$ on CIFAR-100).
This complements the four other settings (ViT BREEDS $46/46$ aligned $+6.53$; CIFAR-100 superclass $13/20$ aligned $+16.15$; NLP SlimPajama $6/6$ fuzzy, $0.10$ PPL behind; LoRA SuperNI $2/15$ fuzzy, tied; all against the matched-architecture control) to characterize the binding condition: the mechanism's contribution is positive on the aligned partitions and null on the fuzzy ones.
We do \emph{not} claim this as evidence of method efficacy in any deployable sense; it is a conceptual anchor for the alignment thesis (Sec.~\ref{sec:discussion}).

\subsection{Information-Matched Logit-Masking Control}
\label{app:logit_mask}

The scope statement of Sec.~\ref{sec:discussion} rests on an information-matched control: at inference, restrict each model's logits to the fine classes of the given category (superclass on CIFAR-100, BREEDS supercategory on ImageNet) and renormalize. This gives every method, including those trained without the label, the identical inference-time information SpecDrop consumes. All rows of Tab.~\ref{tab:logit_mask} are 3-seed means (42/123/456), evaluated on the full test/validation splits ($10{,}000$ and $50{,}000$ images).

\begin{table}[!htbp]
\centering
\small
\caption{\textbf{Logit-masking control} (3 seeds). \emph{Masked} restricts output logits to the given category's fine classes at inference; unmasked entries use the checkpoints of the provenance note below and may differ slightly from Tabs.~\ref{tab:cifar_main}--\ref{tab:vit} (by up to $0.15$). Given the same label, masking the dense model is the strongest deployment for accuracy alone on both datasets; masking changes SpecDrop by exactly $0.00$ on CIFAR and $+1.06$ on ImageNet, showing the output-space restriction is largely internalized during training. On ImageNet the masked No-Routing$+$SE control ($81.44$) also exceeds masked ours ($80.95$); on CIFAR this reverses ($78.47$ vs $79.23$).}
\label{tab:logit_mask}
\begin{tabular*}{\textwidth}{@{\extracolsep{\fill}}llccc@{}}
\toprule
\textbf{Setting} & \textbf{Method} & \textbf{Unmasked} & \textbf{Masked} & \textbf{$\Delta$} \\
\midrule
\multirow{3}{*}{CIFAR-100}
 & Dense ResNet-110        & $74.33$ & $85.23$ & $+10.90$ \\
 & No-Routing              & $63.07$ & $78.47$ & $+15.40$ \\
 & Soft SpecDrop (ours)    & $79.23$ & $79.23$ & $+0.00$ \\
\midrule
\multirow{3}{*}{ImageNet-1K}
 & Dense ViT-S/16          & $76.37$ & $83.65$ & $+7.28$ \\
 & No-Routing $+$ SE       & $73.36$ & $81.44$ & $+8.08$ \\
 & Soft SpecDrop (ours)    & $79.89$ & $80.95$ & $+1.06$ \\
\bottomrule
\end{tabular*}
\end{table}

\paragraph{Internalization statistic.}
The CIFAR $+0.00$ is exact at every seed: across three seeds ($30{,}000$ test predictions), SpecDrop predicts outside the given superclass once ($0/1/0$ per seed), versus $16.3\%$ of the time for the dense model ($16.43/16.32/16.01\%$ per seed), so the mask has nothing left to remove.

\paragraph{Checkpoint provenance.}
Two seed-42 checkpoints were unavailable and retrained from the stored configs before evaluation: the CIFAR dense s42 retrain reaches $73.97$ (original $74.41$; cross-\texttt{torch}-version drift) and the ImageNet dense s42 retrain reaches $76.23$ unmasked (original $76.26$; the retrained seed's masked gain, $+7.28$, matches the other seeds). Table entries use the retrained checkpoints; no conclusion depends on the drift.

\subsection{Label-Quality Curve and Predicted-Label Operating Points}
\label{app:label_quality}

The break-even thresholds of Sec.~\ref{sec:discussion} read off a label-quality $\to$ performance curve: CIFAR-100 top-1 as a function of category-label accuracy under symmetric label corruption (eval-only, 3 seeds), together with realistic predicted-label operating points (the dense baseline's own fine-class prediction coarsened through the hierarchy; no extra training; Tab.~\ref{tab:label_quality}).

\begin{table}[!htbp]
\centering
\small
\caption{\textbf{Label-quality curve} (CIFAR-100, 3 seeds, eval-only re-evaluation of the Tab.~\ref{tab:cifar_main} checkpoints, hence $\pm0.18$ at $100\%$ vs.\ $\pm0.17$ there). Label accuracy under symmetric corruption at rate $p$ is $1 - \tfrac{19}{20}p$ ($5\%$ at $p{=}1$, chance). The \emph{predicted} column is a realistic operating point using the dense model's coarsened fine-class prediction as the label. SpecDrop stays above the architecture-matched No-Routing control ($63.08$) down to ${\approx}80\%$ label accuracy and above dense ($74.48$) down to ${\approx}92$--$94\%$.}
\label{tab:label_quality}
\begin{tabular*}{\textwidth}{@{\extracolsep{\fill}}lcccccccc@{}}
\toprule
Label acc. & $100\%$ & $95.2\%$ & $90.5\%$ & $83.8\%$ (pred.) & $81.0\%$ & $76.3\%$ & $52.5\%$ & $5\%$ \\
\midrule
Top-1 & $79.23$ & $75.37$ & $71.34$ & $69.33$ & $63.14$ & $59.48$ & $41.05$ & $4.02$ \\
$\pm$ std & $0.18$ & $0.13$ & $0.14$ & $0.46$ & $0.17$ & $0.18$ & $0.21$ & $0.09$ \\
\bottomrule
\end{tabular*}
\end{table}

\paragraph{ImageNet operating point.}
The corresponding predictor on ImageNet (the dense ViT's fine-class prediction coarsened through the BREEDS hierarchy) is $88.5\%$ accurate and gives $74.25 \pm 0.10$, retaining $+0.9$ over the matched-SE No-Routing control ($73.36$) under fully predicted labels (on CIFAR the $83.8\%$ predictor retains $+6.2$ over the matched control, $69.3$ vs $63.08$). On CIFAR's finer $5$-classes-per-category partition, predictors at the dense model's own capacity measure only ${\sim}84\%$ (coarsened $83.8\%$; a coarse head fine-tuned from the dense checkpoint reaches ${\sim}84.5\%$), so no capacity-comparable predictor reaches the ${\approx}92$--$94\%$ dense break-even there; the applicability condition of Sec.~\ref{sec:discussion} is load-bearing on such partitions.

\paragraph{Deployment cost.}
Both vision predictors share the same construction: the dense baseline's own fine-class prediction, coarsened through the respective hierarchy (the CIFAR-100 superclass tree; the BREEDS supercategory partition), with no extra training. Where the category tag does not arrive for free (in our NLP and LoRA settings it does, as domain tags and task clusters), predicting it adds a second dense-scale forward pass --- ${\approx}2\times$ the single-model inference cost of App.~\ref{app:flops} --- on CIFAR-100 and ImageNet alike.

\subsection{NLP Per-Domain Reliance Decomposition}
\label{app:nlp_specialization}

\paragraph{Per-domain reliance decomposition (source-domain labels, scalar $p_{\mathrm{a}}{=}0.7$, $w_r{=}0$, no SE, $K{=}7$, 100M tokens, seed $42$).}
Each domain's branch sensitivity is decomposed into ``own-branch'' $\Delta_{k(d),d}$ versus ``other-branches'' $\sum_{k{\neq}k(d)} \Delta_{k,d}$.
The resulting own-reliance ratio anti-correlates with data volume (Book, $4.6\%$ of this subset, has no validation chunks and is omitted):

\begin{center}\small
\begin{tabular*}{\textwidth}{@{\extracolsep{\fill}}lcccr@{}}
\toprule
\textbf{Domain} & \textbf{Data \%} & \textbf{Own-branch $\Delta$} & \textbf{Other-branches $\sum|\Delta|$} & \textbf{Own-reliance} \\
\midrule
CommonCrawl   & $53.4\%$ & $25.82$ & $42.67$ & $38\%$ \\
C4            & $18.5\%$ & $28.60$ & $43.84$ & $39\%$ \\
Github        & $\phantom{0}8.8\%$ & $\phantom{0}5.40$ & $\phantom{0}3.54$ & $60\%$ \\
StackExchange & $\phantom{0}3.8\%$ & $\phantom{0}6.59$ & $\phantom{0}6.48$ & $50\%$ \\
ArXiv         & $\phantom{0}6.6\%$ & $24.35$ & $\phantom{0}5.31$ & $82\%$ \\
Wikipedia     & $\phantom{0}4.3\%$ & $42.78$ & $\phantom{0}7.89$ & $84\%$ \\
\bottomrule
\end{tabular*}
\end{center}
Large web-scrape domains (CommonCrawl, C4) behave as \emph{generalists} that the network's whole branch pool helps represent; smaller, lexically narrower domains (ArXiv, Wikipedia) behave as \emph{specialists} whose assigned branch is nearly solely responsible for their predictions; Github and StackExchange sit between.
The asymmetry is orthogonal to the PPL--specialization decoupling: the model supports a generalist/specialist split internally without translating it to an aggregate PPL gain over the matched-SE control.

\subsection{NLP Mini-Ablation (full)}
\label{app:nlp_miniablation}

Section~\ref{sec:nlp_miniablation} of the main text identifies the final NLP operating point via a three-phase sequential search at 100M tokens (step-wise warmup with $w_r{=}1.0$ throughout, as in the Tab.~\ref{tab:nlp} runs; $K{=}7$ uniform branches, 3 seeds; std in this subsection is the population form, $N$ denominator).
The complete per-cell tables are below.

\needspace{10\baselineskip}
\paragraph{Phase 3a: $p_{\mathrm{a}}$ sweep at SE anchors (step warmup, $\beta{=}1$).}

\begin{center}\small
\begin{tabular*}{\textwidth}{@{\extracolsep{\fill}}lccc@{}}
\toprule
$p_{\mathrm{a}}$ & \textbf{$X{=}0$ step} & \textbf{$X{=}1$ step} & \textbf{$X{=}0.5$ step} \\
\midrule
$0.5$ & $56.37 \pm 0.12$ & $55.08 \pm 0.20$ & $55.01 \pm 0.13$ \\
$0.6$ & $56.45 \pm 0.05$ & $55.14 \pm 0.11$ & --- \\
$0.7$ & $56.91 \pm 0.04$ & $55.54 \pm 0.10$ & --- \\
$0.8$ & $57.81 \pm 0.07$ & $56.25 \pm 0.12$ & --- \\
$0.9$ & $59.49 \pm 0.07$ & $57.46 \pm 0.15$ & --- \\
$1.0$ & $64.32 \pm 0.10$ & $61.37 \pm 0.16$ & --- \\
\bottomrule
\end{tabular*}
\end{center}

The $X{=}0.5$ column reports only the degenerate $p_{\mathrm{a}}{=}0.5$ cell because it serves as the \emph{matched-SE control} against our final operating point.

\paragraph{Phase 3b: $\beta$ sweep at $(p_{\mathrm{a}}, X){=}(0.6, 1)$.}

\begin{center}\small
\begin{tabular*}{\textwidth}{@{\extracolsep{\fill}}lcc@{}}
\toprule
$\beta$ & \textbf{mean PPL} & \textbf{$\sigma_{\text{mean}}$} \\
\midrule
$0$ (scalar form) & $55.17 \pm 0.10$ & $0.056$ \\
$1$ & $55.14 \pm 0.11$ & $0.066$ \\
$2$ & $55.15 \pm 0.18$ & $0.107$ \\
$\mathbf{4}$ (strict argmin) & $\mathbf{55.13 \pm 0.10}$ & $\mathbf{0.057}$ \\
\bottomrule
\end{tabular*}
\end{center}

Spread across $\beta$ is $0.04$ PPL, all within one $\sigma_{\text{mean}}$; $\beta{=}0$ (the scalar form) is the \emph{highest} mean, weak but direction-consistent evidence that per-category differentiation helps at 100M.

\paragraph{Phase 3c: $X$ (SE-ratio) sweep at $(p_{\mathrm{a}}, \beta){=}(0.6, 4)$.}

\begin{center}\small
\begin{tabular*}{\textwidth}{@{\extracolsep{\fill}}lcc@{}}
\toprule
$X$ & \textbf{mean PPL} & \textbf{$\sigma_{\text{mean}}$} \\
\midrule
$0$ & $56.46 \pm 0.03$ & $0.019$ \\
$\mathbf{0.5}$ (strict argmin) & $\mathbf{55.01 \pm 0.10}$ & $\mathbf{0.060}$ \\
$1.0$ & $55.13 \pm 0.10$ & $0.057$ \\
$2.0$ & $55.29 \pm 0.07$ & $0.042$ \\
\bottomrule
\end{tabular*}
\end{center}

Clean U-curve with minimum at $X{=}0.5$, the configuration reported in Table~\ref{tab:nlp}.

\paragraph{100M matched-SE tie, seed-paired.}
Ours at the final operating point gives $55.01\pm 0.10$ (s=42: 55.15, s=123: 54.98, s=456: 54.90); the matched-SE control gives $55.01\pm 0.13$ (s=42: 55.20, s=123: 54.93, s=456: 54.90).
Seed-paired differences $(-0.05, +0.05, 0.00)$ sum to zero; $z{=}0.00\sigma$, exactly tied.
The $+0.10$ PPL gap at 500M (Table~\ref{tab:nlp}) is within seed noise and does not appear at 100M.

\paragraph{Warmup-granularity: step vs epoch at 100M.}
A side comparison at $(p_{\mathrm{a}}, \beta, X){=}(0.6, 2, 0.5)$ compares per-step vs per-epoch $p_{\mathrm{a}}$ warmup at 100M.
Mean PPL differs by $-0.03$ (step wins, within noise); 3-seed $\sigma$ drops from $0.15$ (epoch) to $0.08$ (step), a $2\times$ reduction.
At the degenerate $p_{\mathrm{a}}{=}0.5$ point, step and epoch are mathematically identical; we verify numerically that the residual drift between the epoch-warmup and step-warmup Phase 3a runs is at most $0.09$ PPL across the $p_{\mathrm{a}}{\leq}0.9$ cells, within the $\sigma_{\text{mean}}{\leq}0.12$ envelope; the largest drift is at $p_{\mathrm{a}}{=}0.5$ itself, where only CUDA non-determinism acts (the $p_{\mathrm{a}}{=}1.0$ cell, farthest from the degenerate point, drifts $+0.59$).

\subsection{ViT ImageNet Mini-Ablation}
\label{app:vit_miniablation}

The ViT operating point of Sec.~\ref{sec:vit} is identified by a three-phase sequential search on a $20\%$ stratified ImageNet subset (per-class), validated on the full $50$K validation split, $K{=}46$ uniform branches, $w_r{=}1.0$ cosine per-epoch warmup, $3$ seeds (Phase 5a--5c std: population form, $N$ denominator).

\paragraph{Phase 5a: $p_{\mathrm{a}}$ sweep at $X{=}1$, $\beta{=}1$.}

\begin{center}\small
\resizebox{\linewidth}{!}{%
\begin{tabular}{@{}lcccccc@{}}
\toprule
$p_{\mathrm{a}}$ & $0.5$ (routing off) & $0.6$ & $0.7$ & $0.8$ & $0.9$ & $1.0$ \\
\midrule
Top-1 (\%) & $42.91 \pm 0.18$ & $\mathbf{55.52 \pm 0.15}$ & $55.30 \pm 0.21$ & $55.21 \pm 0.13$ & $54.94 \pm 0.19$ & $52.75 \pm 0.42$ \\
\bottomrule
\end{tabular}}
\end{center}

$p_{\mathrm{a}}{=}0.5$ excluded \emph{a priori} (algebraic mechanism-OFF). Strict argmax at $p_{\mathrm{a}}{=}0.6$.

\paragraph{Phase 5b: $\beta$ sweep at $p_{\mathrm{a}}{=}0.6$, $X{=}1$.}

\begin{center}\small
\begin{tabular*}{\textwidth}{@{\extracolsep{\fill}}lcccc@{}}
\toprule
$\beta$ & $0$ & $\mathbf{1}$ (argmax) & $2$ & $4$ \\
\midrule
Top-1 (\%) & $55.31 \pm 0.10$ & $\mathbf{55.52 \pm 0.15}$ & $55.29 \pm 0.32$ & $55.40 \pm 0.25$ \\
\bottomrule
\end{tabular*}
\end{center}

$\beta$ is a flat lever (spread $0.23$ within seed noise), $\beta{=}1$ strict argmax.

\paragraph{Phase 5c: $X$ (SE-ratio) sweep at $(p_{\mathrm{a}}, \beta){=}(0.6, 1)$.}

\begin{center}\small
\begin{tabular*}{\textwidth}{@{\extracolsep{\fill}}lcccc@{}}
\toprule
$X$ & $0$ & $0.5$ & $1.0$ & $\mathbf{2.0}$ (argmax) \\
\midrule
Top-1 (\%) & $53.40 \pm 0.35$ & $54.54 \pm 0.21$ & $55.52 \pm 0.15$ & $\mathbf{56.46 \pm 0.26}$ \\
\bottomrule
\end{tabular*}
\end{center}

$X$ is the only ViT axis with significant signal near the optimum: $X{=}2.0$ provides $+0.94$ over $X{=}1.0$, indicating BREEDS' moderate semantic dispersion benefits from a $2\times$-capacity shared expert.
The full-data ViT main-table number ($79.89$, Table~\ref{tab:vit}) at $(p_{\mathrm{a}}, \beta, X){=}(0.6, 1, 2.0)$ is consistent with the $20\%$-subset selection transferring to full ImageNet-1K (only the selected point was run on full data at 3 seeds).

\paragraph{Mini-ablation decomposition.}
Phases 5a/5b/5c jointly decompose the contribution of each axis on the BREEDS partition.
Starting from $(p_{\mathrm{a}}, \beta, X){=}(0.5, 1, 1)$ (mechanism-OFF, $42.91$ top-1), three findings emerge.
\emph{(i)~Mechanism necessity}: $p_{\mathrm{a}}{=}0.5\to 0.6$ is $\mathbf{+12.61}$ top-1 ($42.91\to 55.52$), a discrete jump from routing off ($p_{\mathrm{a}}{=}p_{\mathrm{i}}$) to on with architecture and SE held constant.
This is the cleanest single-table evidence that the routing mechanism itself contributes, not just the $K{=}46$ multi-branch architecture --- on ViT BREEDS the mechanism is responsible for the bulk of the gain over architecture-matched baselines (cf.\ Tab.~\ref{tab:vit} headline $+6.53$ over matched-SE at full ImageNet).
\emph{(ii)~$\beta$-plateau}: $\beta{\in}\{0,1,2,4\}$ all within $0.23$ top-1 (statistically tied): despite BREEDS-46's class-count imbalance, the per-category amplification term shows no measurable leverage here, as on SlimPajama and SuperNI (App.~\ref{app:nlp_miniablation}, \ref{app:lora_miniablation}).
\emph{(iii)~SE-monotonic}: $X{\in}\{0,0.5,1,2\}$ is monotonic, $\mathbf{+3.06}$ from $X{=}0$ to $X{=}2$, a secondary boost.
Combined: on the imbalanced ViT BREEDS partition, the routing mechanism contributes the bulk of the gain, SE provides a secondary monotonic boost, and per-category amplification is statistically inactive despite its motivation as an imbalance-correction term.

\paragraph{ImageNet-BREEDS label-leak disclosure.}
The BREEDS supercategory is derived from each image's fine-grained ImageNet-1K label~\citep{santurkar2021breeds}. Our routing observes a label-derived signal that no baseline in Tab.~\ref{tab:vit} observes, including the No-Routing controls. The matched-SE No-Routing baseline (uniform $1/K$ + SE on the same $K{=}46$ branches) is the architecture- and compute-matched mechanism-OFF reference: the $+6.53$ measures routing on the label at fixed architecture and includes the label's value, which the masking control (App.~\ref{app:logit_mask}) separates.

\paragraph{$20\%{\to}100\%$ subset transfer.}
Sec.~\ref{sec:vit} reports the final operating point $(p_{\mathrm{a}}, \beta, X){=}(0.6, 1, 2.0)$ at full ImageNet-1K, identified on the $20\%$ stratified subset.
We acknowledge that per-subset argmax may differ from full-data argmax; the full-data result for ours $79.89\pm 0.18$ is reported.

\paragraph{DeiT short-recipe disclosure.}
We train all ViT methods with a shortened DeiT recipe: 100~epochs, no model EMA, no RepeatedAugmentation. Full DeiT 300-ep + EMA + RA reaches $\sim 79.85\%$ for ViT-S/16 (timm reference); our short-recipe dense ViT-Small reaches $76.38\%$. All methods share this recipe identically; the full DeiT 300-ep+EMA+RA recipe is out of scope for this work.

\paragraph{Mod-Squad FFN-only adaptation.}
The original Mod-Squad~\citep{chen2023modsquad} targets multi-task vision (Taskonomy, PASCAL-Context) with MoE applied to both attention and FFN. We adapt it to single-task ImageNet by treating each token as its own ``task'' in the mutual-information loss (no BREEDS label is used), and we restrict MoE to FFN blocks only ($16$ experts of hidden $96$ per block, top-$2$; the original uses top-$4$); the original paper's ablation suggests this FFN-only restriction underestimates Mod-Squad by $\sim 0.3$--$1$ top-1. We accept this asymmetry rather than re-tune.

\paragraph{Soft MoE tuning study (placement, learning rate, granularity, compute).}
\label{app:softmoe_tuning}
The deployed all-blocks Soft MoE configuration reaches $63.06 \pm 0.22$ at full protocol; a dedicated tuning study shows this number is a placement artifact, and Table~\ref{tab:vit} accordingly reports the tuned variant as the primary Soft MoE row. On the same $20\%$ stratified selection subset used for our own operating-point search above, we sweep eight Soft MoE variants plus the ALF router and a capacity-identical Mod-Squad pair, every row at 3 seeds:

\begin{center}\small
\begin{tabular*}{\textwidth}{@{\extracolsep{\fill}}lc@{}}
\toprule
\textbf{Configuration (20\% subset, 3 seeds)} & \textbf{Top-1 (\%)} \\
\midrule
Soft SpecDrop reference                                        & $56.46 \pm 0.32$ \\
\textbf{Soft MoE, canonical 2nd-half placement $+$ lr $5{\times}10^{-4}$} & $\mathbf{52.63 \pm 0.31}$ \\
Soft MoE, canonical 2nd-half placement only                    & $51.41 \pm 0.50$ \\
ALF top-$k$ router~\citep{wang2024alf}                         & $46.66 \pm 0.23$ \\
Mod-Squad (capacity-identical pair)                            & $44.43 \pm 0.31$ \\
Soft MoE, lr $5{\times}10^{-4}$ only                           & $39.38 \pm 0.11$ \\
Soft MoE, $16$ experts $\times$ $96$-dim                       & $37.29 \pm 0.82$ \\
Soft MoE, deployed config (all blocks)                         & $37.27 \pm 0.21$ \\
Soft MoE, $8$ experts $\times$ $192$-dim                       & $36.89 \pm 0.29$ \\
Soft MoE, compute-matched ($\approx$ dense MACs; $2.7\times$ deployed) & $35.57 \pm 0.37$ \\
Soft MoE, lr $1{\times}10^{-4}$                                & $31.25 \pm 0.57$ \\
\bottomrule
\end{tabular*}
\end{center}

One change dominates: the paper-canonical second-half placement is worth $+14.1$ on the subset by itself ($51.41$ vs $37.27$), the learning rate adds $+1.2$ on top ($52.63$ vs $51.41$), and no other single change moves the number by more than $+2.1$; the lr $1{\times}10^{-4}$ and compute-matched variants hurt ($-6.0$ and $-1.7$). (The reference row's $\pm 0.32$ is a sample standard deviation; the same three runs appear as $56.46 \pm 0.26$ in Phase 5c above under its population convention.) Carried to the full protocol (identical to Tab.~\ref{tab:vit}: full data, $100$ epochs, 3 seeds; ALF, like its subset Mod-Squad pair, accumulates two micro-batches of $128$ per step, whereas the Tab.~\ref{tab:vit} Mod-Squad row uses single batches of $256$, the same effective batch), the tuned combination reaches $76.69 \pm 0.70$ (per-seed $75.92/77.30/76.86$; the seed-456 configuration was accidentally scheduled twice, and we report the run whose checkpoint and results artifacts are retained, $76.86$ --- the displaced duplicate's log records $76.73$, a $0.13$ same-seed replication gap), above dense; the compute-matched variant lands at $66.72 \pm 0.63$, below the bare No-Routing control while consuming dense-level compute, so the binding constraint of the deployed configuration was placement, not compute; and the ALF router reaches $71.09 \pm 0.32$ vs the Tab.~\ref{tab:vit} Mod-Squad row's $70.11$ ($+0.98$ at full protocol, where the two differ in micro-batching; $+2.23$ on the fully matched subset pair; consistent in direction), so the bias-corrected balancing helps, while both remain below the $73.36$ matched-SE control. The sweep gave Soft MoE a placement and learning-rate search that no other method in Table~\ref{tab:vit}, including ours, received.

\begin{figure}[!htbp]
\centering
\includegraphics[width=0.72\linewidth]{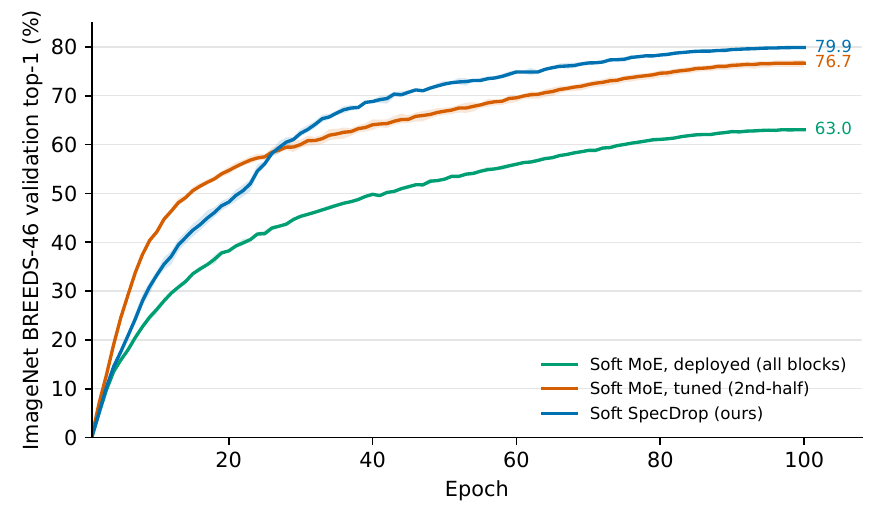}
\caption{\textbf{ImageNet BREEDS-46 validation top-1 across training} for the deployed all-blocks Soft MoE, the tuned Soft MoE (canonical second-half placement $+$ lr $5{\times}10^{-4}$), and Soft SpecDrop. Lines are 3-seed means; shaded bands span seed min/max. Milestones (3-seed means) at epochs $10/25/50/75/100$: deployed $26.3/41.8/52.9/60.0/63.0$, tuned $42.1/57.5/66.8/73.6/76.7$, ours $33.3/56.1/72.4/77.5/79.9$.}
\label{fig:training_curves}
\end{figure}

\paragraph{Compute-matched and ALF configurations.}
The compute-matched Soft MoE widens experts from $48$ to $1600$ hidden dimensions at $5$ slots per expert, with parameters unconstrained ($481$M), restoring dense-level per-image compute ($4.26$ GMACs vs dense $4.25$; the deployed all-blocks variant runs at $1.56$ GMACs, $37\%$ of dense). The ALF router follows \citet{wang2024alf}: top-$2$ selection over $N{=}16$ experts of hidden $96$ (capacity-identical to our Mod-Squad configuration) with bias-corrected, auxiliary-loss-free load balancing (bias update rate $0.001$). Per-method MACs for all rows are in App.~\ref{app:flops}.

\paragraph{Training curves (deployed vs tuned vs ours).}
Figure~\ref{fig:training_curves} plots validation top-1 across training for the deployed Soft MoE, the tuned Soft MoE, and Soft SpecDrop (3-seed means from per-epoch histories). All three converge healthily under the shared $100$-epoch budget that every Table~\ref{tab:vit} method received. The tuned baseline dominates the deployed one at every epoch, consistent with the placement finding; it also leads Soft SpecDrop through epoch $25$ before Soft SpecDrop overtakes it by epoch $50$, and the ordering is stable across the final quarter of training (epoch-$75$ gap $3.9$ vs final $3.2$).

\subsection{LoRA SuperNI Mini-Ablation}
\label{app:lora_miniablation}

The LoRA operating point of Sec.~\ref{sec:lora} is identified by a three-phase sequential search on a $20\%$ per-task subset of the SuperNI training instances (every training task kept), $K{=}20$ uniform branches, $w_r{=}1.0$ cosine per-step warmup, $3$ seeds, ROUGE-L F1 selection metric end-to-end.

\paragraph{Phase 8a: $p_{\mathrm{a}}$ sweep at $X{=}1$, $\beta{=}1$.}
$p_{\mathrm{a}}{\in}\{0.5,0.6,0.7,0.8,0.9,1.0\}$ all within $0.0030$ ROUGE-L; $p_{\mathrm{a}}{=}0.5$ (mechanism-OFF) excluded \emph{a priori} and $p_{\mathrm{a}}{=}1.0$ ($0.4786\pm 0.0108$) after the sweep; strict argmax of the remaining values at $p_{\mathrm{a}}{=}0.8$ ($0.4784\pm 0.0090$).
The flat $p_{\mathrm{a}}$-curve is itself diagnostic: the routing axis is muted in this regime, consistent with the small routing $\Delta$ ($+0.0012$) in Tab.~\ref{tab:se_routing_decomp}.

\paragraph{Phase 8b: $\beta$ sweep at $(p_{\mathrm{a}}, X){=}(0.8, 1)$.}
$\beta{\in}\{0,1,2,4\}$ all within $0.0026$ ROUGE-L; $\beta{=}1$ strict argmax (no per-category amplification benefit).

\paragraph{Phase 8c: $X$ sweep at $(p_{\mathrm{a}}, \beta){=}(0.8, 1)$.}

\begin{center}\small
\begin{tabular*}{\textwidth}{@{\extracolsep{\fill}}lcccc@{}}
\toprule
$X$ & $0$ & $0.5$ & $\mathbf{1.0}$ (argmax) & $2.0$ \\
\midrule
ROUGE-L & $0.4709 \pm 0.0050$ & $0.4728 \pm 0.0104$ & $\mathbf{0.4784 \pm 0.0090}$ & $0.4766 \pm 0.0073$ \\
\bottomrule
\end{tabular*}
\end{center}

The only LoRA axis with a visible, if weak, signal (${\sim}1\sigma$ between $X{=}0$ and $X{=}1$).
Strict argmax at $X{=}1.0$, consistent with full-data Table~\ref{tab:lora}.

\paragraph{ROUGE-L F1 selection metric.}
We use ROUGE-L F1, the canonical Tk-Instruct metric of \citet{wang2022supernaturalinstructions}, end-to-end: the reported checkpoint is the epoch with the highest ROUGE-L (per-epoch generation eval), and the selected $p_{\mathrm{a}}/\beta/X$ maximize the 3-seed mean ROUGE-L among the non-excluded values.
A single metric across (1) checkpoint selection, (2) hyperparameter selection, and (3) baseline comparison ensures consistency.

\paragraph{Cluster ID at inference.}
SpecDrop-LoRA reads the cluster ID from each batch at inference; for SuperNI held-out tasks every task carries an official SuperNI \emph{Domains} label, so this is unambiguous. Table~\ref{tab:lora} covers all $119$ held-out tasks; domains unseen in training map to the miscellaneous cluster (App.~\ref{app:superni_clusters}).

\paragraph{$225$M / $18\%$ trainable budget justification.}
Our LoRA configuration ($K{=}20$ branches attached to all $7$ linear projections in each transformer block; rank $r{=}16$ for the no-SE variant, $r{=}15$ plus a rank-$15$ shared expert as deployed) yields $\approx 225$M trainable parameters, $\approx 18\%$ of the Llama-3.2-1B base. This budget is substantially larger than single-LoRA defaults (typically $<$1\% at $r{=}8$).
All seven methods of Tab.~\ref{tab:lora} are budget-matched to $225$M ($\pm 3\%$); the regime is non-typical, and method comparisons are made within it.

\paragraph{LoRAMoE rank-extrapolation disclosure.}
LoRAMoE's native rank range~\citep{dou2024loramoe} is $r{\in}\{4, 8, 16\}$; our budget-matching constraint forces $r{=}76$ on Llama-3.2-1B. This is well outside the original paper's tested range and may dilute its load balancing. In place of the localized balancing constraint of \citet{dou2024loramoe}, our LoRAMoE uses a conventional balance loss (squared coefficient of variation of the $6$ experts' batch-mean gate weights, summed over the adapted FFN linears, weight $0.1$). We did not run a comparison at native $r{=}8$ (outside our $225$M budget).

\paragraph{HydraLoRA rank-extrapolation disclosure.}
HydraLoRA's main result~\citep{tian2024hydralora} (their Table~2, LLaMA-2-7B) uses $N{=}3$ B-heads at rank $r{=}8$, with $N{=}4$ reported as empirically optimal in their ablation (Section~4.5 and Figure~8, $N \in \{1, \ldots, 5\}$). We use $N{=}8$ at rank $r{=}67$ on Llama-3.2-1B to match our $225$M trainable-parameter budget; rank-matched at $r{=}8$ would yield only $\sim 29$M trainable ($\sim 13\%$ of our budget) and starve HydraLoRA's capacity. We do not run Tian's $k$-means step, which selects $N$; our router is zero-initialized (the LoRA factors keep the standard Kaiming~$A$, zero~$B$) and all heads share one dropout mask. $N{=}8$, the larger rank, and these choices are our deviations from the canonical setup.

\paragraph{LoRAMoE site-coverage confound.}
\citet{dou2024loramoe} attach LoRAMoE to FFN sites only ($3$ linears per block), and so does our LoRAMoE baseline ($r{=}76$, budget-matched), whereas ours, HydraLoRA, and MoCLE attach to all $7$ linears. A budget-matched $7$-site LoRAMoE variant ($r{=}53$) ran out of memory on a $32$\,GB GPU at the per-device batch of $8$ shared by all LoRA methods and was not rerun at a smaller batch. The single-baseline coverage difference is acknowledged.

\paragraph{LoRA $7$-linear attachment + GQA + MoCLE adaptations.}
All multi-branch LoRA methods except LoRAMoE (FFN-only, above) attach adapters to all $7$ linear projections (q/k/v/o + gate/up/down) per transformer block, regardless of original-paper defaults (the single LoRA of \citet{hu2022lora} attaches to q/v only; attaching to all linears follows the QLoRA recommendation). Llama-3.2-1B uses Grouped-Query Attention: q/o\_proj have dim $2048$, but k/v\_proj have dim $512$ ($8$ KV heads $\times$ $64$ head\_dim); a given rank is therefore $4\times$ larger relative to k/v\_proj than to q/o\_proj. This does not affect param budgeting (all methods share base dimensions), but we report it for completeness.
MoCLE's original $64$ gate clusters are reduced to our $20$ clusters to match our cluster partition (App.~\ref{app:superni_clusters}), with the $5$ experts ($E{=}4{+}1$ universal) accessed by all $20$ cluster IDs through a learned $20$-entry cluster embedding and a softmax gate over the $4$ task experts masked to its top-$1$ weight $G_{\max}$, with the universal expert weighted $1{-}G_{\max}$ (consistent with the ablation of \citet{gou2024mocle}, where raising the cluster count from $16$ to $64$ changes average performance only marginally).

\subsection{LoRA Aggregated Per-Method Results (3 seeds)}
\label{app:lora_per_seed}

Aggregated ROUGE-L F1 and Exact-match (mean$\pm$std over seeds $42$/$123$/$456$) for the SuperNI main-table comparison (Sec.~\ref{sec:lora}, Table~\ref{tab:lora}); per-seed JSON artifacts are available on request.

\begin{center}\small
\begin{tabular*}{\textwidth}{@{\extracolsep{\fill}}lcc@{}}
\toprule
\textbf{Method} & \textbf{ROUGE-L F1 (mean$\pm$std)} & \textbf{Exact-match (mean$\pm$std)} \\
\midrule
HydraLoRA $N{=}8\,r{=}67$               & $0.5153 \pm 0.0034$ & $0.3482 \pm 0.0059$ \\
\textbf{Soft SpecDrop (ours)}            & $0.5106 \pm 0.0032$ & $0.3417 \pm 0.0017$ \\
MB-LoRA No-Routing $+$ SE ($X{=}1$)      & $0.5094 \pm 0.0072$ & $0.3378 \pm 0.0080$ \\
LoRAMoE $K{=}6\,r{=}76$                  & $0.5079 \pm 0.0022$ & $0.3350 \pm 0.0026$ \\
MB-LoRA No-Routing (no SE)               & $0.4993 \pm 0.0106$ & $0.3291 \pm 0.0130$ \\
MoCLE $E{=}4{+}1\,r{=}63$                & $0.4924 \pm 0.0103$ & $0.3289 \pm 0.0056$ \\
Single LoRA $r{=}320$                    & $0.4754 \pm 0.0072$ & $0.3196 \pm 0.0064$ \\
\bottomrule
\end{tabular*}
\end{center}

\subsection{NLP Scaling Check ($125$M Transformer)}
\label{app:scaling_check}

We replicate the 30M training regime at $125$M scale (GPT-2-small architecture: $12$ layers, hidden $768$, $12$ heads; $500$M unique SlimPajama tokens $\times$ $10$ epochs; AdamW $\text{lr}\,3{\times}10^{-4}$ cosine, bf16, batch $16 \times 512$ (halved from $30$M batch $32 \times 512$ due to memory), $3$ seeds $42/123/456$; the table's std is the sample form) to test whether the $30$M ours-vs-matched-SE tie persists across $\sim$$4\times$ scale.
We chose this regime over a Chinchilla-optimal $1$-epoch run at $125$M to preserve internal cross-scale comparability; between $30$M and $125$M the model size and the per-device batch ($32 \to 16$) change, and the $125$M configurations supply no category frequencies, so ours uses the scalar form of Eq.~\ref{eq:per_category} (one $(p_{\mathrm{a}}, p_{\mathrm{i}})$ for all domains, i.e.\ no $\beta$ amplification; $\beta$ spans $0.04$ PPL at $100$M, App.~\ref{app:nlp_miniablation}).

\begin{center}\small
\begin{tabular*}{\textwidth}{@{\extracolsep{\fill}}lcc@{}}
\toprule
\textbf{Method} ($125$M, 3 seeds) & \textbf{Val PPL} & per-seed (s42/s123/s456) \\
\midrule
\textbf{Soft SpecDrop $+$ SE (ours, $p_{\mathrm{a}}{=}0.6$, no $\beta$, $X{=}0.5$, step)} & $33.99 \pm 0.05$ & $33.96 / 34.05 / 33.96$ \\
MB-LM No-Routing $+$ SE ($X{=}0.5$, matched)                                    & $33.82 \pm 0.10$          & $33.77 / 33.93 / 33.75$ \\
\midrule
$\Delta$ (ours $-$ matched) at $125$M     & $\mathbf{+0.17}$ & ${+}0.19 / {+}0.12 / {+}0.21$ \\
\bottomrule
\end{tabular*}
\end{center}

\paragraph{Reading.}
At $125$M, ours and the matched-SE control both drop ${\sim}25\%$ in PPL relative to $30$M, confirming both benefit from scale. The $30$M direction (ours $0.10$ PPL behind, $3$ seeds) is preserved at $125$M (ours $0.17$ behind, in all $3$ paired seeds), consistent with the granularity-alignment thesis predicting no differential routing gain on this fuzzy partition at either scale.

\subsection{NLP $1$-Epoch Regime Sanity ($30$M)}
\label{app:nlp_regime_sanity}

To verify that the $10$-epoch multi-epoch regime (App.~\ref{app:nlp_regime_disclosure}) does not bias the cross-method orderings of Tab.~\ref{tab:nlp}, we rerun all $8$ Tab.~\ref{tab:nlp} methods at $30$M $\times$ $1$ epoch on the same $500$M-unique SlimPajama cache ($\approx 17$ tokens per parameter, Chinchilla-near-optimal), $3$ seeds each ($24$ cells), identical to the main configuration except for epoch count and batch size ($64$ vs.\ $32$ sequences).

\begin{center}\small
\begin{tabular*}{\textwidth}{@{\extracolsep{\fill}}lcccc@{}}
\toprule
\textbf{Method} & \textbf{$1$-ep PPL $\pm \sigma$} & \textbf{$10$-ep PPL $\pm \sigma$} & \textbf{$\Delta$ ($1$ep$-10$ep)} & \textbf{rank @ $10$ep} \\
\midrule
Dense                                & $58.38 \pm 0.16$ & $44.80 \pm 0.05$ & $+13.58$ & $1$ \\
\textbf{Soft SpecDrop (ours)}        & $61.91 \pm 0.13$ & $45.38 \pm 0.02$ & $+16.53$ & $3$ \\
\textbf{No-Routing $+$ SE} (matched) & $61.93 \pm 0.17$ & $45.28 \pm 0.10$ & $+16.65$ & $2$ \\
No-Routing (no SE)                   & $64.29 \pm 0.15$ & $46.80 \pm 0.11$ & $+17.49$ & $4$ \\
Switch                               & $63.38 \pm 0.45$ & $49.54 \pm 0.20$ & $+13.84$ & $5$ \\
Hash Layers                          & $66.00 \pm 0.19$ & $52.05 \pm 0.06$ & $+13.95$ & $6$ \\
DEMix                                & $68.40 \pm 0.36$ & $53.31 \pm 0.08$ & $+15.09$ & $7$ \\
SMoE-Dropout ($k{=}1$)               & $81.49 \pm 0.64$ & $67.32 \pm 0.50$ & $+14.17$ & $8$ \\
\bottomrule
\end{tabular*}
\end{center}
The $1$-epoch column reports the sample std ($N{-}1$); the $10$-epoch column repeats Tab.~\ref{tab:nlp} (population std).

\paragraph{Reading (regime-invariant tie).}
The headline ours-vs-matched-SE $\Delta$ shrinks from $+0.10$ PPL ($10$-ep, $1.7\sigma$, $3/3$ seeds lose; $\sigma{=}\sqrt{(\sigma_1^2{+}\sigma_2^2)/3}$ is the standard error of the difference of the two 3-seed means, from each method's population std) to $\mathbf{-0.02}$ PPL ($1$-ep, $0.23\sigma$, $2/3$ seeds win) --- sign-flipped but \emph{still tied within seed noise}.
Cross-method orderings are largely preserved: ours and matched-SE swap positions $2\leftrightarrow 3$ (within $0.02$ PPL at $1$-ep), and Switch and No-Routing (no SE) swap positions $4\leftrightarrow 5$ (No-Routing gains more from the extra epochs, $+17.49$ vs.\ $+13.84$ PPL); Hash, DEMix, and SMoE-Dropout retain their $10$-ep ranks.
Within multi-branch, ours and matched-SE move \emph{together} ($+16.53$ vs.\ $+16.65$ PPL higher at $1$-ep than at $10$-ep), confirming the routing mechanism does not differentially benefit from extra epochs.
The SlimPajama tie is therefore not an artifact of the multi-epoch training regime.

\subsection{LoRA Pruning-Sensitivity Diagnostic}
\label{app:lora_diag}

Per-cluster $\times$ per-branch zero-ablation on Soft SpecDrop at all three seeds (summary statistics below: seed $42$, the SuperNI column of Fig.~\ref{fig:specialization}), with the $K{=}20$ clusters built from the SuperNI \emph{Domains} field (App.~\ref{app:superni_clusters}), evaluated on $15$ test-split clusters (5 clusters had no held-out tasks), at the routing weights with which each seed's reported checkpoint was scored.

\paragraph{Summary statistics.}
Diagonal hits $2/15$ (clusters $4$ and $19$; $8$ of the $15$ rows tie, up to floating-point rounding, for their most-negative $\Delta$, broken by lowest branch index; Fig.~\ref{fig:specialization} shades $|\Delta|$, so in clusters $11$ and $18$, where every ablation raises ROUGE-L, the darkest cell is the assigned branch but not a hit); $\max|\Delta|{=}0.081$ ROUGE-L (cluster $15$ $\times$ branches $10$, $13$ and $19$, tied, all off-diagonal); mean $|\Delta|{=}0.0096$; diag mean $\Delta{=}{+}0.0013$, off-diag mean $\Delta{=}{-}0.0019$; sign skew $121$ positive / $60$ zero / $119$ negative (of $300$ cells; zero up to floating-point rounding).

\paragraph{Reading.}
Over three seeds the diagonal hit rate is $22.2\pm15.4\%$ ($2$, $6$ and $2$ of $15$; Tab.~\ref{tab:lora}): above the $5\%$ chance level on average (individually distinguishable from chance only at seed $123$; under chance, at least $2$ of $15$ hits occur with probability $0.17$), but unstable across seeds and far below the aligned vision settings (Tabs.~\ref{tab:cifar_main}--\ref{tab:vit}) and SlimPajama ($94.4\%$).
At seeds $123$ and $456$ the assigned branch's mean $\Delta$ is more negative than the other branches' (${-}0.0126$ vs ${-}0.0003$; ${-}0.0138$ vs ${-}0.0053$); at seed $42$ it is not.
Pruning sensitivity therefore detects at most weak branch structure aligned with the imposed $K{=}20$ \emph{Domains}-field partition, consistent with the LoRA $\Delta{\approx}0$ tie under uniform-mask inference (Sec.~\ref{sec:specialization}).
Whether a different $K$ or task clustering recovers stronger alignment is left open.

\subsection{LoRA Per-Task ROUGE-L Breakdown}
\label{app:lora_per_task}

Seed $42$, $119$ SuperNI held-out tasks, $10$ instances per task (the ROUGE-L subsample of Tab.~\ref{tab:lora}), three LoRA methods.
Mean ROUGE-L: ours $0.5114$, HydraLoRA $0.5117$, MB-LoRA No-Routing (no SE) $0.4962$ (each within about $1\sigma$ of its 3-seed main-table mean).

\paragraph{Sorted-$\Delta$ split (ours $-$ HydraLoRA per task).}
$47$ tasks ours wins, $47$ HydraLoRA wins, $25$ tied --- balanced split, mean $\Delta{\approx}0.000$.

\emph{Right tail (ours wins):}
\begin{itemize}\small
\item \texttt{task1390\_wscfixed\_coreference} ($\Delta{=}{+}0.40$)
\item \texttt{task202\_mnli\_contradiction\_classification} ($\Delta{=}{+}0.40$)
\item \texttt{task1158\_bard\_analogical\_reasoning\_manipulating\_items} ($\Delta{=}{+}0.40$)
\item \texttt{task033\_winogrande\_answer\_generation} ($\Delta{=}{+}0.30$)
\item \texttt{task036\_qasc\_topic\_word\_to\_generate\_related\_fact} ($\Delta{=}{+}0.21$)
\end{itemize}

\emph{Left tail (HydraLoRA wins):}
\begin{itemize}\small
\item \texttt{task1387\_anli\_r3\_entailment} ($\Delta{=}{-}0.50$)
\item \texttt{task233\_iirc\_link\_exists\_classification} ($\Delta{=}{-}0.50$)
\item \texttt{task200\_mnli\_entailment\_classification} ($\Delta{=}{-}0.50$)
\item \texttt{task020\_mctaco\_span\_based\_question} ($\Delta{=}{-}0.30$)
\item \texttt{task1152\_bard\_analogical\_reasoning\_causation} ($\Delta{=}{-}0.20$; tied with three other tasks)
\end{itemize}

The tails are listed descriptively; App.~\ref{app:lora_fisher_per_task} finds no significant association between the win pattern and task cluster.

\subsection{Per-Task Win-Pattern \texorpdfstring{$\times$}{x} Cluster Association (\texorpdfstring{$\chi^2$}{chi-squared} on \texorpdfstring{$n{=}119$}{n=119} tasks)}
\label{app:lora_fisher_per_task}

We extend the granularity-alignment test from $n{=}4$ settings to $n{=}119$ tasks by cross-tabulating the per-task ours-vs-HydraLoRA win/loss/tie (App.~\ref{app:lora_per_task}) with the $K{=}20$ \emph{Domains}-field cluster ID and applying a $\chi^2$ omnibus test plus per-cluster Fisher exact tests.

\begin{center}\small
\begin{tabular*}{\textwidth}{@{\extracolsep{\fill}}lccc@{}}
\toprule
\textbf{Test} & \textbf{Statistic} & \textbf{dof} & \textbf{$p$-value} \\
\midrule
$\chi^2$ (ours-win vs other $\times$ cluster) & $18.33$ & $14$ & $0.19$ \\
$\chi^2$ (ours / HydraLoRA / tie $\times$ cluster) & $28.77$ & $28$ & $0.42$ \\
Fisher exact, per-cluster ($15$ clusters tested)     & ---     & --- & min $0.048$ \\
\bottomrule
\end{tabular*}
\end{center}

Both omnibus tests fail to reject $H_0$ (random allocation); the smallest per-cluster Fisher $p$ is $0.048$ (cluster $3$, $n{=}11$), not significant after a Bonferroni correction for the $15$ tests.
Per-cluster heterogeneity exists descriptively (cluster $13$, $n{=}5$, ours wins $80\%$; cluster $3$, $n{=}11$, HydraLoRA wins $73\%$) but is not statistically distinguishable from chance at the available per-cluster sample sizes (median per-cluster $n{=}5$; $\chi^2$ with small expected counts loses power).

\paragraph{Consistency with the pruning-sensitivity diagnostic.}
Neither diagnostic finds a reliable association with the imposed partition: per-cluster zero-ablation finds at most weak, seed-unstable branch structure aligned with it (App.~\ref{app:lora_diag}), and the per-task win pattern shows no significant association with the $K{=}20$ \emph{Domains}-field clusters.
At fine task grain, this evidence is consistent with --- but does not statistically prove --- the granularity-alignment thesis behind the cross-setting trend.

\subsection{Embedding-Structure Diagnostic: BGE vs DINOv2 Modality Asymmetry}
\label{app:nlp_clusters}

To characterize whether SlimPajama text chunks carry the discrete cluster structure that categorical routing presupposes, we run a 3-metric cluster-validity scan (silhouette, Calinski-Harabasz (CH), Davies-Bouldin (DB)) on BGE-large-en-v1.5 embeddings of $195$K $512$-token SlimPajama chunks over $k \in [2, 50]$, with DINOv2-base embeddings of CIFAR-100 images ($50$K samples) as a contrastive image-modality reference.

\paragraph{BGE embeddings lack discrete cluster structure.}
The BGE scan gives no consensus optimum: silhouette argmax at $k{=}3$ with value $0.031$, far below the Kaufman-Rousseeuw $0.25$ ``substantial structure'' threshold (DINOv2's $0.069$, reported below, is also under it, so the contrast is relative); CH near-monotone-decreasing; DB argmin at $k{=}47$ via singleton-artifact.
DINOv2 embeddings of CIFAR-100 images give silhouette $s_{\max}{=}0.069$ at $k{=}50$ with healthy clusters (minimum cluster size $364$, no singletons).
BGE text-chunk embeddings of web text form a \emph{continuous manifold}; DINOv2 image embeddings are comparatively more clustered (higher silhouette, no singleton clusters).
The modality asymmetry is consistent with the granularity-mismatch interpretation of Sec.~\ref{sec:discussion}: categorical routing is not a universal mechanism but one that presupposes a categorizable data modality.

\paragraph{Full-corpus intra-chunk mixture and per-chunk correlation.}
For all $9{,}766$ SlimPajama validation chunks we compute (i) the intra-chunk topic mixture, slicing each $512$-token chunk into $8$ sub-windows of $64$ tokens, BGE-embedding each sub-window, and assigning it to the $k{=}7$ train-fit clusters; and (ii) the per-chunk cross-entropy difference $\Delta\text{CE}$ between the matched-SE No-Routing control and ours ($3$ seeds each).
$\mathbf{56.1\%}$ of chunks span ${\geq}2$ clusters, quantifying the partition's fuzziness.
The per-chunk correlation between purity and $\Delta\text{CE}$ is null (Pearson $r{=}{-}0.001$, $p{=}0.95$; Spearman $\rho{=}{-}0.001$, $p{=}0.94$), and homogeneous chunks ($43.9\%$ of the set) versus mixed chunks show statistically indistinguishable mean $\Delta\text{CE}$ ($-0.0023$ vs $-0.0021$).
This null is the expected signature of a \emph{training-time, distribution-level} property rather than an inference-time, per-chunk one: branch specialization forms over the whole training distribution, so a chunk being homogeneous at evaluation does not retroactively give it a specialized branch.
The evidence for the granularity-alignment thesis is therefore cross-setting --- partitions that are clean throughout training (CIFAR, BREEDS) yield gains, fuzzy ones yield no gain --- robust to training protocol (App.~\ref{app:nlp_regime_sanity}) and scale (App.~\ref{app:scaling_check}).

\subsection{Per-Seed and Per-Domain Results}
\label{app:per_seed_domain}

Tables~\ref{tab:all_seeds}--\ref{app:nlp_per_domain} report per-seed top-1/PPL for the CIFAR and NLP main tables and per-domain PPL for three representative methods.

\begin{table}[h]
\centering
\small
\caption{Per-seed top-1 accuracy (\%) on CIFAR-100 for the baseline comparison of Table~\ref{tab:cifar_main} except HardCategory; MultiBranchResNet110 / dense ResNet-110, $\sim$1.7M params, 200 epochs, RTX 5090.}
\label{tab:all_seeds}
\begin{tabular*}{\textwidth}{@{\extracolsep{\fill}}lcccc@{}}
\toprule
\textbf{Method} & \textbf{Seed 42} & \textbf{Seed 123} & \textbf{Seed 456} & \textbf{Mean $\pm$ Std} \\
\midrule
ResNet-110 (dense)        & 74.41 & 74.40 & 74.63 & $74.48 \pm 0.13$ \\
Stochastic Depth          & 75.91 & 75.88 & 75.62 & $75.80 \pm 0.16$ \\
Example-Tied Dropout      & 64.43 & 62.58 & 64.02 & $63.68 \pm 0.97$ \\
Contextual Dropout        & 70.19 & 70.55 & 70.00 & $70.25 \pm 0.28$ \\
No-Routing (equal weights)& 63.10 & 63.03 & 63.11 & $63.08 \pm 0.04$ \\
\textbf{Soft SpecDrop (ours)} & \textbf{79.03} & \textbf{79.31} & \textbf{79.35} & $\mathbf{79.23 \pm 0.17}$ \\
\bottomrule
\end{tabular*}
\end{table}

\begin{table}[h]
\centering
\small
\caption{Per-seed validation perplexity on SlimPajama-6B for the NLP comparison of Table~\ref{tab:nlp} (30M-parameter Transformer LM, 500M tokens, 10 epochs, RTX 5090, three seeds 42/123/456; population std, as in Table~\ref{tab:nlp}). The matched No-Routing+SE seed-456 value $45.14$ is the lowest of the three matched-SE seeds; per-seed (ours $-$ matched-SE) PPL differences are ${+}0.05/{+}0.02/{+}0.22$ at s42/s123/s456, so the aggregate $+0.10$ PPL gap (ours behind matched-SE) is dominated by s456.}
\label{tab:nlp_all_seeds}
\begin{tabular*}{\textwidth}{@{\extracolsep{\fill}}lcccc@{}}
\toprule
\textbf{Method} & \textbf{Seed 42} & \textbf{Seed 123} & \textbf{Seed 456} & \textbf{Mean $\pm$ Std} \\
\midrule
Dense Transformer        & 44.79 & 44.74 & 44.87 & $44.80 \pm 0.05$ \\
No-Routing $+$ SE ($X{=}0.5$, matched) & 45.36 & 45.34 & 45.14 & $45.28 \pm 0.10$ \\
\textbf{Soft SpecDrop (ours)} & 45.41 & 45.36 & 45.36 & $45.38 \pm 0.02$ \\
No-Routing               & 46.75 & 46.95 & 46.71 & $46.80 \pm 0.11$ \\
Switch Transformer       & 49.81 & 49.44 & 49.36 & $49.54 \pm 0.20$ \\
Hash Layers              & 52.13 & 52.00 & 52.01 & $52.05 \pm 0.06$ \\
DEMix                    & 53.27 & 53.24 & 53.42 & $53.31 \pm 0.08$ \\
SMoE-Dropout             & 67.91 & 66.69 & 67.37 & $67.32 \pm 0.50$ \\
\bottomrule
\end{tabular*}
\end{table}

\begin{table}[h]
\centering
\small
\caption{Per-domain validation perplexity on SlimPajama-6B (seed 42, 500M tokens, 10 epochs) for the three methods shown. Absolute per-domain difficulty varies by an order of magnitude across domains. Book domain ($4.1\%$ of training data) is omitted: the validation set ($9{,}766$ chunks) contains no Book chunks.}
\label{app:nlp_per_domain}
\begin{tabular*}{\textwidth}{@{\extracolsep{\fill}}lcccccc@{}}
\toprule
\textbf{Method} & \textbf{CC} & \textbf{C4} & \textbf{Github} & \textbf{ArXiv} & \textbf{Wiki} & \textbf{Stack} \\
\midrule
Dense                  & 62.84 & 67.85 &  6.02 & 11.59 & 37.71 & 12.93 \\
Switch                 & 70.27 & 75.47 &  6.52 & 12.60 & 42.68 & 14.07 \\
Hash Layers            & 73.97 & 79.16 &  6.69 & 13.10 & 43.24 & 14.55 \\
\bottomrule
\end{tabular*}
\end{table}

\FloatBarrier
\section{Architecture and Implementation Details}
\label{app:arch}

Routing enters through the branch merge, which is applied at every routed layer with the same per-sample branch weights: all three layer groups of MultiBranchResNet110, the MLPs of all $12$ MB-ViT blocks, the FFN of every LM block, and all $7$ adapted linears in each of Llama-3.2-1B's $16$ blocks.

\subsection{MultiBranchResNet110 Architecture}

For CIFAR-100 experiments at ResNet-110 scale, we use \texttt{MultiBranchResNet110} where all three layer groups are branched, mirroring the every-layer MoE design of modern LLMs~\citep{jiang2024mixtral}:
\[
\text{conv1 (shared)} \to K\!\times\!\text{layer1} \to \text{merge} \to K\!\times\!\text{layer2} \to \text{merge} \to K\!\times\!\text{layer3} \to \text{merge} \to \text{FC}
\]
With \texttt{num\_blocks}$\,{=}\,18$, this gives ResNet-110 depth (3 groups $\times$ 18 blocks $\times$ 2 convs + 2 = 110 layers).
Branch channel widths are auto-computed to match single-branch ResNet-110 parameter count ($\sim$1.74M): with $K\!{=}\!20$, branch channels are $[4, 7, 14]$ per layer group.
The always-on branch (the shared expert of Eq.~\ref{eq:soft_specdrop}: never masked, added after the $\div S$ merge with unit weight) uses the same channel widths as routed branches but may have a different block count, controlling capacity without requiring dimension projection.
This is a natural extension of the routing framework---not a separate architecture---and is parameter-matched by reducing routed branch widths accordingly.
The design is inspired by the shared expert in DeepSeekMoE~\citep{dai2024deepseekmoe}.

\subsection{Efficient Multi-Branch Computation}
\label{app:efficiency}

A naive implementation of $K$ parallel branches uses a Python \texttt{for}-loop, launching $K$ sequential CUDA kernels per layer.
With $K\!{=}\!20$ branches, 3 layer groups, 18 blocks per group, and the standard ResNet \texttt{BasicBlock}'s 2 convolutions per block, this results in $3 \times 20 \times 18 \times 2 = 2{,}160$ sequential kernel launches per forward pass---over an order of magnitude slower than a single-branch ResNet-110 on a modern GPU.

\paragraph{CNN: Grouped convolution.}
We fuse all $K$ branch convolutions into a single grouped convolution (\texttt{groups}$\,{=}\,K$).
The shared input is repeated along the channel dimension: $(B, C, H, W) \to (B, KC, H, W)$.
Each group processes its $C$ input channels independently with its own filter set, producing $(B, KC_{\text{out}}, H, W)$ in one CUDA kernel call.
\texttt{BatchNorm}$(K\,C_{\text{out}})$ naturally provides per-branch normalization since each group of $C_{\text{out}}$ output channels has independent statistics ($\gamma$, $\beta$, running mean/var).
This reduces kernel launches from $2{,}160$ to $\sim 108$ (two grouped convs per block, $54$ blocks total).

\paragraph{NLP: Batched einsum.}
For transformer FFN branches, we stack $K$ weight matrices into tensors $\mathbf{W}_1 \in \mathbb{R}^{K \times F \times D}$ and $\mathbf{W}_2 \in \mathbb{R}^{K \times D \times F}$, then compute all branches simultaneously via \texttt{torch.einsum}:
\[
\begin{aligned}
\mathbf{H} &= \text{GELU}(\texttt{einsum}(\text{`btd,kfd}\to\text{btkf'}, \mathbf{X}, \mathbf{W}_1) + \mathbf{b}_1),\\
\mathbf{O} &= \texttt{einsum}(\text{`btkf,kdf}\to\text{btkd'}, \mathbf{H}, \mathbf{W}_2) + \mathbf{b}_2
\end{aligned}
\]
This replaces $K$ sequential matrix multiplications with two batched operations.

\paragraph{Additional optimizations.}
We use mixed-precision training (\texttt{torch.amp}) and \texttt{torch.compile} for kernel fusion.
Combined with grouped convolution, these achieve \textbf{11.0$\times$ speedup} over the naive for-loop implementation, reducing per-batch time from 1,387ms to 126ms (measured on an NVIDIA A100 during development).

\paragraph{Equivalence verification.}
Both optimizations are mathematically equivalent to the naive implementation:
(1)~forward outputs match within \texttt{atol{=}1e-5} on CPU (\texttt{tests/test\_grouped\_conv.py});
(2)~gradient differences pass a \texttt{1e-3} threshold (typical observed magnitude $<10^{-4}$, the threshold accommodates cuDNN nondeterminism);
(3)~a 20-epoch training comparison yields final accuracy within seed noise (verified by \texttt{archive/tests/test\_grouped\_training\_equivalence.py}).

\begin{table}[h]
\centering
\small
\caption{Computational overhead of multi-branch architectures vs.\ single-branch baselines. All times are forward+backward per batch, measured on an NVIDIA A100 during implementation development, on development configurations rather than the deployed ones (CV: $K{=}20$ branches with channels $[4,7,14]$ plus a full-depth always-on branch, batch $128$; NLP: $6$-layer LM, hidden $256$, $4$ heads, $7{\times}512$ branch FFNs plus a $512$ shared expert, batch $8{\times}256$ tokens, dense reference with a $3584$-wide FFN); the speedup columns compare implementations on the same device; only the ``$+$ AMP $+$ compile'' row uses mixed precision and \texttt{torch.compile}. End-to-end training wall-clock on the RTX 5090 GPUs used for all reported runs is reported in App.~\ref{app:flops}.}
\label{tab:efficiency}
\begin{tabular*}{\textwidth}{@{\extracolsep{\fill}}llrrr}
\toprule
\textbf{Setting} & \textbf{Method} & \textbf{Time (ms)} & \textbf{vs Single} & \textbf{vs For-loop} \\
\midrule
\multirow{4}{*}{CV ($K{=}20$)} & ResNet-110 & 68.5 & 1.0$\times$ & --- \\
& MultiBranch (for-loop) & 1,387 & 20.2$\times$ & 1.0$\times$ \\
& + Grouped conv & 198 & 2.9$\times$ & 7.0$\times$ \\
& + AMP + compile & 126 & 1.8$\times$ & 11.0$\times$ \\
\midrule
\multirow{3}{*}{NLP ($K{=}7$)} & Dense Transformer & 25.4 & 1.0$\times$ & --- \\
& MoE+SE (for-loop) & 44.0 & 1.7$\times$ & 1.0$\times$ \\
& MoE+SE (einsum) & 31.3 & 1.2$\times$ & 1.4$\times$ \\
\bottomrule
\end{tabular*}
\end{table}

\subsection{Per-Method Compute (MACs) and Wall-Clock}
\label{app:flops}

We report per-forward compute (Tab.~\ref{tab:macs}) for every method in the four main tables as \emph{multiply--accumulate operations} (MACs): one MAC is one multiplication plus one addition, so $\text{FLOPs} \approx 2 \times \text{MACs}$. CIFAR, ImageNet and SlimPajama MACs are counted with \texttt{fvcore}'s operator-level counter, which skips the attention products as our models compute them: the ImageNet counts omit the query--key and attention--value products of PyTorch's fused attention kernel (${\approx}0.36$G per image for every ViT variant; dense ViT-S/16 is ${\approx}4.61$G with them) and the SlimPajama counts omit the query--key product (${\approx}0.60$G per sequence for all eight methods); SuperNI adapter MACs are computed in closed form. SpecDrop's fixed routing adds zero MACs (no router network); we claim no compute advantage.

\begin{table}[!htbp]
\centering
\small
\caption{\textbf{Per-method MACs across the four settings.} CIFAR: per $32{\times}32$ image; ImageNet: per $224{\times}224$ image; SlimPajama: per $512$-token sequence; SuperNI: adapter add-on per token as \% of the frozen Llama-3.2-1B base (${\sim}1.24$ GMACs/token, shared by all methods), counting every adapter the forward computes (gate linears excluded). The deployed all-blocks Soft MoE runs at $37\%$ of dense MACs under parameter matching; the tuned second-half variant at $68\%$; the compute-matched variant restores dense-level MACs with parameters unconstrained. MoE MACs are counted as implemented: the ImageNet Mod-Squad and ALF baselines and the SlimPajama Switch, Hash, SMoE-Dropout and DEMix baselines compute every expert and then gather the routed output. On ImageNet, ours and No-Routing$+$SE are profiled at the deployed SE ratio $X{=}2$ ($46{\times}32{+}64 = 1536$ hidden units, dense's FFN width); at $X{=}1$ both would be $4.19$G.}
\label{tab:macs}
\begin{tabular*}{\textwidth}{@{\extracolsep{\fill}}llr@{}}
\toprule
\textbf{Setting} & \textbf{Method} & \textbf{MACs} \\
\midrule
\multirow{2}{*}{CIFAR-100}
 & dense-backbone baselines & 255.3M \\
 & MultiBranch $K{=}20$ (No-Routing, HardCategory, ours) & 287.6M ($+12.7\%$) \\
\midrule
\multirow{6}{*}{ImageNet-1K}
 & Dense ViT-S/16 & 4.25G \\
 & Soft SpecDrop (ours) / No-Routing$+$SE & 4.25G ($\pm 0.0\%$) \\
 & No-Routing (no SE) & 4.22G \\
 & Mod-Squad / ALF top-$k$ router & 4.27G \\
 & COMET & 5.65G \\
 & Soft MoE: deployed / tuned / comp.-matched & 1.56G / 2.91G / 4.26G \\
\midrule
SlimPajama & all eight methods, $15.93$--$15.97$G & within $0.3\%$ \\
\midrule
\multirow{4}{*}{SuperNI/LoRA}
 & Single LoRA / No-Routing & $+18.2\%$ \\
 & Soft SpecDrop (ours) / No-Routing$+$SE & $+18.0\%$ \\
 & LoRAMoE / MoCLE & $+18.1\%$ / $+18.0\%$ \\
 & HydraLoRA & $+18.1\%$ \\
\bottomrule
\end{tabular*}
\end{table}

\begin{table}[!htbp]
\centering
\small
\caption{\textbf{Measured wall-clock, all four settings} (3-seed mean total training hours on RTX 5090, bf16 mixed precision except fp16 with loss scaling on CIFAR-100; ratios vs the per-setting dense/single reference). Ours is within ${\sim}0.5\%$ of the architecture-matched No-Routing control on three of four settings and within $7\%$ on CIFAR: the routing rule adds little to no measurable cost; the cost is the multi-branch architecture, shared by ours and the No-Routing controls. The three additional ImageNet baselines (bottom) were trained on identical-model GPUs on a separate node; their last column gives seconds per epoch (not a ratio), directly comparable to the deployed Soft MoE's $502$\,s (ALF accumulates two micro-batches of $128$ per step, App.~\ref{app:vit_miniablation}).}
\label{tab:wallclock_full}
\begin{tabular*}{\textwidth}{@{\extracolsep{\fill}}llcc@{}}
\toprule
\textbf{Setting} & \textbf{Method} & \textbf{Wall (h)} & \textbf{$\times$ ref} \\
\midrule
\multirow{6}{*}{CIFAR-100 (200 ep)}
 & ResNet-110 (ref) & $0.67$ & $1.00\times$ \\
 & Example-Tied Dropout & $0.74$ & $1.11\times$ \\
 & No-Routing & $0.95$ & $1.43\times$ \\
 & \textbf{ours} & $1.01$ & $1.51\times$ \\
 & Contextual Dropout & $1.13$ & $1.70\times$ \\
 & Stochastic Depth & $1.98$ & $2.96\times$ \\
\midrule
\multirow{7}{*}{ImageNet ViT (100 ep)}
 & Soft MoE (deployed) & $13.95$ & $1.00\times$ \\
 & ViT-S/16 (ref) & $14.00$ & $1.00\times$ \\
 & COMET & $15.95$ & $1.14\times$ \\
 & No-Routing$+$SE & $27.30$ & $1.95\times$ \\
 & \textbf{ours} & $27.44$ & $1.96\times$ \\
 & No-Routing (no SE) & $28.66$ & $2.05\times$ \\
 & Mod-Squad & $58.41$ & $4.17\times$ \\
\midrule
\multirow{8}{*}{SlimPajama (10 ep)}
 & Dense Transformer (ref) & $5.12$ & $1.00\times$ \\
 & No-Routing & $5.42$ & $1.06\times$ \\
 & No-Routing$+$SE (matched) & $6.06$ & $1.18\times$ \\
 & \textbf{ours} & $6.08$ & $1.19\times$ \\
 & DEMix & $7.07$ & $1.38\times$ \\
 & Hash Layers & $7.24$ & $1.42\times$ \\
 & SMoE-Dropout & $7.78$ & $1.52\times$ \\
 & Switch & $8.21$ & $1.60\times$ \\
\midrule
\multirow{7}{*}{SuperNI/LoRA (3 ep)}
 & Single LoRA $r{=}320$ (ref) & $3.68$ & $1.00\times$ \\
 & LoRAMoE & $7.69$ & $2.09\times$ \\
 & MoCLE & $9.28$ & $2.52\times$ \\
 & HydraLoRA & $10.41$ & $2.83\times$ \\
 & No-Routing (no SE) & $14.70$ & $3.99\times$ \\
 & No-Routing$+$SE & $16.33$ & $4.43\times$ \\
 & \textbf{ours} & $16.39$ & $4.45\times$ \\
\midrule
\multirow{3}{*}{Additional ImageNet baselines}
 & Soft MoE (tuned, 2nd-half) & $13.36$ & $481$\,s/ep \\
 & Soft MoE (compute-matched) & $19.71$ & $710$\,s/ep \\
 & ALF top-$k$ router & $64.87$ & $2335$\,s/ep \\
\bottomrule
\end{tabular*}
\end{table}

\end{document}